\documentclass[accepted]{uai2026} 
                        
\usepackage[british]{babel}

\usepackage{natbib} 
\usepackage{mathtools} 
\usepackage{booktabs} 

\usepackage{microtype}

\usepackage{amssymb}
\usepackage{amsthm}
\usepackage[tighter]{newtxtext}
\usepackage[slantedGreek]{newtxmath} 

\usepackage[capitalize,noabbrev]{cleveref}
\AddToHook{cmd/appendix/before}{\crefalias{section}{appendix}}

\Crefname{assumption}{Assumption}{Assumptions}

\usepackage{graphicx}
\usepackage{subcaption}
\usepackage{url}
\usepackage{xurl} 
\usepackage{multirow}
\usepackage{enumitem}
\usepackage{pifont} 
\usepackage{cancel}

\usepackage{thmtools}

\theoremstyle{plain}

\theoremstyle{definition}
\declaretheorem[name=Definition]{definition}
\declaretheorem[name=Assumption]{assumption}

\theoremstyle{remark}

\declaretheorem[name=Example]{example}

\usepackage{xcolor}

\usepackage{tikz}
\usetikzlibrary{positioning,graphs,arrows.meta}
\usetikzlibrary{trees} 
\usetikzlibrary{matrix}
\usepackage[normalem]{ulem} 

\usepackage[ruled,vlined,linesnumbered]{algorithm2e}
\SetKwComment{Comment}{/*}{*/}
\SetKwInput{Inv}{Invariant}
\crefname{algocf}{alg.}{algs.}
\Crefname{algocf}{Algorithm}{Algorithms}
\usepackage{float} 
\usepackage{wrapfig}
\usepackage{braket}

\usepackage[acronym]{glossaries}
\usepackage{glossaries-extra}
\setkeys{glslink}{hyper=false}

\newacronym{dag}{DAG}{directed acyclic graph}
\newacronym{ug}{UG}{undirected graph}
\newacronym{pgm}{PGM}{probabilistic graphical model}
\newacronym{ci}{CI}{conditional independence}
\newacronym{gs}{GS}{Grow and Shrink}
\newacronym{mb}{MB}{Markov blanket}
\newacronym{bn}{BN}{Bayesian network}
\newacronym{mrf}{MRF}{Markov random field}

\hypersetup{hidelinks}

\usepackage{amsmath,amsfonts,bm}

\def\eqref#1{equation~\ref{#1}}

\def\1{\bm{1}}

\def\db{{\mathcal{D}}}
\def\pr{{\mathbb{P}}}
\def\pre{{\hat{\mathbb{P}}}}

\def\method{{\emph{kOMB}}}
\newcommand{\perpperp}{\Perp}
\newcommand{\indeps}[1]{\mathcal{J}(#1)}
\newcommand{\indep}[1]{\perpperp_{#1}}
\newcommand{\dep}[1]{\:\cancel{\perpperp}_{#1}\:}
\newcommand{\kassoc}[3]{\textit{asc}_{#1}(#2\mid #3)}

\newcommand{\lassoc}[3]{\textit{asc}^{(l)}_{#1}(#2\mid #3)}
\newcommand{\kinter}[2]{\textit{asc}_{#1}(#2)}

\newcommand{\kdep}[3]{\textit{dep}_{#1}(#2\mid #3)}

\newcommand{\power}[1]{\mathcal{P}_{#1}}
\newcommand{\powerleq}[2]{\mathcal{P}_{\leq #1}(#2)}
\newcommand{\powereq}[2]{\mathcal{P}_{= #1}(#2)}

\newcommand{\subverts}{\verts\setminus\{Y,X\}}
\newcommand{\mbs}[1]{\text{MB}_{#1}}

\newcommand{\model}{\mathcal{M}}
\newcommand{\gr}{\mathcal{G}}
\newcommand{\verts}{\bm{V}}
\newcommand{\edges}{\bm{E}}

\def\rvx{{\mathbf{x}}}

\def\rmS{{\mathbf{S}}}

\def\rmV{{\mathbf{V}}}
\def\rmW{{\mathbf{W}}}
\def\rmX{{\mathbf{X}}}
\def\rmY{{\mathbf{Y}}}
\def\rmZ{{\mathbf{Z}}}

\DeclareMathAlphabet{\mathsfit}{\encodingdefault}{\sfdefault}{m}{sl}
\SetMathAlphabet{\mathsfit}{bold}{\encodingdefault}{\sfdefault}{bx}{n}

\title{High-Order Markov Blanket Discovery\\
  via a k-Order Relaxation of the Faithfulness Assumption}

\author[1]{\href{mailto:<loong.kuan.lee@iais.fraunhofer.de>?Subject=Your UAI 2026 paper}{Loong~Kuan~Lee}{}}
\author[1]{Ragavi~Krishnamoorthy}
\author[1]{Nico~Piatkowski}
\affil[1]{%
  Hybrid Intelligence\\
  Fraunhofer IAIS\\
  Sankt Augustin, Germany\\
}

\begin{document}
\maketitle

\begin{abstract}
  The problem of learning the graphical Markov blanket (MB) of a
  variable from data has applications in many areas such as structure
  learning for Bayesian networks and Markov random fields, causal
  discovery, and feature selection. However, a common assumption most
  methods make is that the conditional independencies in the
  distribution imply the same separation in the graphical
  structure---also known as the faithfulness assumption. Unfortunately,
  this assumption can be violated by higher-order dependencies such as
  XOR and parity-type relations, and---on finite samples---by empirical
  violations that, in extreme cases, even induce spurious dependencies
  absent from the true distribution. Therefore, in this
  paper we propose a ``k-order'' relaxation of the faithfulness
  assumption that captures parity type relationships between k+2
  variables. We then propose a proof of concept algorithm called k-order
  Markov blanket (kOMB) that uses this relaxation for MB
  discovery. Finally, we empirically show how kOMB can recover the MB of
  a variable under both true and empirical violations of faithfulness.
  Code available at:
  \url{https://github.com/lklee9/k-order-Markov-blanket}.
\end{abstract}

\section{Introduction}\label{sec:intro}
The goal of the \acrfull{mb} discovery problem is to infer the \gls{mb}
of some variable $Y$ in the graph $\gr$ given a dataset $\db$ sampled
from the joint distribution $\pr$.  Originating from the Bayesian
network and causal learning
literature~\citep{margaritis1999,tsamardinos2003,ling2021,hu2024}, the
MB discovery problem has also been discussed and utilised in the feature
selection literature as well~\citep{tsamardinos2003a,wang2020}. More
recently, \gls{mb} discovery has also been employed as a subroutine
within distributed, divide-and-conquer approaches to large-scale
\gls{bn} structure learning~\citep{dong2025a}. That said, we are
interested in inferring \glspl{mb} in both \acrfull{bn} and
\acrfull{mrf}. Therefore, whenever possible we will try to keep our
notation and exposition agnostic to the graph type of $\gr$.

\begin{figure}
  \centering
  \begin{tikzpicture}
  \node[draw, circle, inner sep=2pt] (1) at (-2.5,3) {$X_{1}$};
  \node[draw, circle, inner sep=2pt] (2) at (-1.5,3) {$X_{2}$};
  \node[draw, circle, inner sep=2pt] (3) at (-0.5,3) {$X_{3}$};
  \node[draw, circle, inner sep=3pt] (y) at (-1.5,1.75) {$Y$};
  \draw[-Stealth] (1) -- (y);
  \draw[-Stealth] (2) -- (y);
  \draw[-Stealth] (3) -- (y);
  \node (Z) at (-1.5, 1) {$Z:=\Bigl[\sum_{i=1}^{3}X_{i}=1\Bigr]$};
  \node (prY) at (-1.5, 0.25) {$\pr(Y=Z\mid \rmX)=0.9$};

\node (PrY) at (2.75, 1.75){
      \setlength{\tabcolsep}{3pt}
    \begin{tabular}[t]{lcc}
      \multicolumn{3}{c}{Recall (\%)}\\
      \toprule
      & \multicolumn{2}{c}{Sample Size}\\
      Method & 100 & 1000 \\
      \midrule
      Grow and Shrink & 28.33 & 100\\
      0-OMB (ours)& 28.33 & 100\\
      1-OMB (ours)& 77.50 & 100\\
      \textbf{2-OMB (ours)} & \textbf{100.0} & 100\\
      \bottomrule
    \end{tabular}
};
\end{tikzpicture}
\caption{ \centering Recovering the \gls{mb} of variable $Y$
  when $Y$ is the output of a noisy Boolean function that returns $1$
  when only one of $X_{1}$, $X_{2}$, and $X_{3}$ is $1$. Our algorithm
  that assumes a $2$-order faithfulness relaxation is able to fully
  recover the \gls{mb} of each variable---even with just $100$
  samples.}\label{fig:intro}
\end{figure}
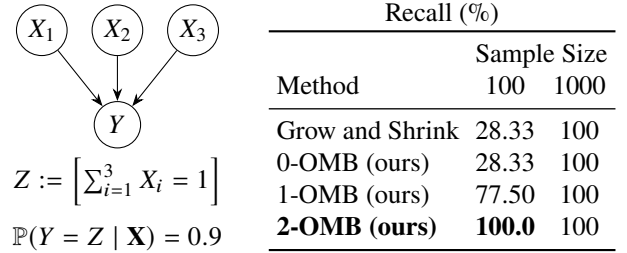

Most existing approaches to MB discovery operate under two main
assumptions. The first being that any separations in $\gr$ imply the
same conditional independence in $\pr$. The second assumption---known as
the faithfulness assumption---goes the opposite direction, that is any
conditional independencies in $\pr$ imply the same separation in $\gr$.

Theoretically, \glspl{bn} with faithfulness violations have Lebesgue
measure zero~\citep{meek1995,spirtes2001,boeken2025}; however these
``typicality'' results depend on somewhat arbitrary choices, like the
choice of $\sigma$-ideal~\citep{boeken2025}. There are well-known
examples of some---very simple---relations that violate
faithfulness. The most famous of which is the (noisy) XOR function for
violations involving three variables~\citep{inazumi2011,marx2021}; which
is a special case of the (noisy) parity function whose higher-order
faithfulness violations we detail in \Cref{ex:parity}.

That said, even if unfaithful \glspl{bn} and \glspl{mrf} are truly
almost non-existent in nature, a lack of sufficient samples can still
cause an empirical faithfulness
violation~\citep{uhler2013,lemeire2012,boeken2025}. For example, the
small example in \Cref{fig:intro} is one such case where the true
\gls{bn} is faithful, but due to a finite sample, empirical faithfulness
violations occur---which can be overcome by the algorithm we will
propose later in \Cref{sec:algo}.  A full walkthrough of empirical
faithfulness violation can be found later in \Cref{ex:sample} under
\Cref{sec:violation}.  Furthermore, in extreme cases, a limited sample
size can cause \gls{mb} discovery algorithms to falsely discover
spurious dependencies absent from the true distribution. We later observe
this phenomenon in \Cref{fig:syn-metrics} under \Cref{sec:exp-syn} with
the low precision exhibited by the different \gls{mb} discovery methods
at low sample sizes.

There are several weaker assumptions compared to faithfulness such as
adjacency faithfulness~\citep{spirtes2001},
Pearl-minimality~\citep{pearl1988,pearl2009},
SGS-minimality~\citep{spirtes2001}, frugality~\citep{forster2018}, and
more recently $2$-adjacency faithfulness~\citep{marx2021}.
$2$-adjacency faithfulness in particular is quite practical as by
limiting faithfulness violations to ``Unfaithful Triples'', it yields an
algorithm with a much smaller search space compared to
frugality. However, although the $2$-adjacency faithfulness can find
XOR-type relations, it is unable to find higher-order parity
relationships between more than $3$ variables.

\paragraph{Contributions}
Therefore, motivated by the considerations above, in \Cref{sec:relax} we
propose a relaxation of faithfulness for parity-like relationships
between $k+2$ variables: the $k$-order faithfulness assumption. This
relaxation of faithfulness allows us to infer---from sample
data---the graphical \acrlong{mb} of a variable under mild
assumptions. We show this in \Cref{sec:algo} by implementing a
proof-of-concept algorithm called the $k$-order Markov blanket
(kOMB). We then empirically show in \Cref{sec:exp} that kOMB can
overcome both true and empirical faithfulness violations and that it
outperforms some existing constraint-based methods on well-known
benchmark datasets---where we find merit in exploring relationships
between $4$ variables at a time, and not just $3$.


\section{Probabilistic Graphical Models and Independencies}
Consider the set of $n$ variables $\verts=\{V_{1},\ldots,V_{n}\}$ with
no other hidden or latent variables. The conditional
independencies between these variables---$\rmY \indep{} \rmX \mid \rmZ$
for disjoint variables sets $\rmY,\rmX,\rmZ\subseteq\verts$---can be
represented by either a \acrfull{dag}~\citep{pearl1988} or an
\acrfull{ug}~\cite{lauritzen1996}.
That said, we will denote both graphs with the tuple
$\gr=(\verts,\edges)$ since we are agnostic on the type of graph used
for representing these graphical conditional independencies
$\indeps{\gr}$.




Regardless of the graph type of $\gr$, for disjoint subsets
$\rmX,\rmY,\rmS\subseteq\verts$, we will denote that $\rmX$ and $\rmY$
are separated in $\gr$ as, $\rmX \indep{\gr} \rmY \mid \rmS$.
Furthermore, we denote the set of conditional independencies encoded by
the graph $\gr$---i.e. the independence model of $\gr$---as,
\begin{equation}
  \label{eq:g-indeps}
  \braket{\rmX,\rmY\mid\rmS} \in \indeps{\gr}
  \iff \rmX \indep{\gr} \rmY \mid \rmS,
\end{equation}
where $\braket{\rmX,\rmY\mid\rmS}$ is known as an independence
statement~\citep{sadeghi2017}.

In addition to the graphical structure $\gr$, we also consider a joint
distribution $\pr$ over $\verts$ that might encode its own set of
conditional independencies between the variables in $\verts$,
\begin{equation}
  \label{eq:4}
  \braket{\rmX,\rmY\mid\rmS} \in \indeps{\pr}
  \iff \rmX \indep{\pr} \rmY \mid \rmS,
\end{equation}
where $\rmX \indep{\pr} \rmY \mid \rmS$ denotes $\rmX$ and
$\rmY$ are probabilistically conditionally independent given $\rmS$.

So far, $\indeps{\gr}$ and $\indeps{\pr}$ are two separate collections
of conditional independencies in the graph $\gr$ and distribution $\pr$
respectively. Our base assumption throughout this paper, the
\emph{Global Markov Property}, links the two by requiring that every
conditional independence encoded by the graph $\gr$ also hold in the
distribution $\pr$. This is a common assumption in the Markov blanket
discovery literature~\cite{schluter2014,yu2020a,marx2021}.

\begin{assumption}[Global Markov Property]\label{def:global-markov}
  Given a \gls{pgm} $\model=(\gr,\pr)$, we say that the distribution
  $\pr$ is \emph{Markov} to the graph $\gr$ when for all disjoint
  subsets
  $\rmX,\rmY,\rmS\subseteq\verts$~\citep{hammersley1971,lauritzen2018},
  \begin{equation}
    \label{eq:5}
    \rmX \indep{\gr} \rmY \mid \rmS \implies
    \rmX \indep{\pr} \rmY \mid \rmS.
  \end{equation}
  Furthermore, by \Cref{eq:g-indeps,eq:4}, the relation in \Cref{eq:5}
  is equivalent to $\indeps{\gr} \subseteq \indeps{\pr}$.
\end{assumption}
Despite their differences, the independence models of $\gr$ and $\pr$
are both known to obey a set of axioms called the (semi)-graphoid
axioms~\citep{pearl1989} (see \Cref{def:graphoid}). As mentioned in
\Cref{sec:intro}, we are interested in learning the graphical
\acrlong{mb} of some $Y\in\verts$ defined as follows.
\begin{definition}[Markov Blanket and Boundary] \label{def:mb} For a
  variable $Y\in\verts$, the \gls{mb} of $Y$ in $\gr$ is any subset
  $\rmS \subseteq \rmV\setminus Y$\footnote{We allow set operations
    between variable sets and single variables to implicitly mean the
    operation between the set and the atomic set of the single
    variable---for example:
    $\rmV \setminus Y \coloneq \rmV \setminus \{Y\}$.}  s.t.
  \begin{equation}
    \label{eq:def-mb}
    Y \indep{\gr} \verts \setminus (Y \cup \rmS) \mid \rmS,
  \end{equation}
  or in other words, every other variable in
  $\verts \setminus (Y \cup \rmS)$ is graphically separated from $Y$ by
  $\rmS$.  The Markov boundary of the variable $Y$---henceforth denoted
  as $\mbs{\gr}(Y)$---is then defined as the minimal Markov blanket of
  $Y$~\citep{pearl1988}.
  
  An equivalent definition for the Markov blanket and boundary of $Y$
  in $\pr$ can be made by just substituting the graph separation term
  in \Cref{eq:def-mb} for conditional independence in $\pr$.
\end{definition}
Therefore, it is more accurate to say that we wish to learn the Markov
boundary of $Y$ and not just its Markov blanket. However, since
colloquially the term Markov blanket is commonly used to implicitly
refer to the Markov boundary, we shall use the term Markov blanket in
such a fashion as well.


Several methods exist for discovering the \acrfull{mb} of a target
variable $Y\in\verts$ in $\gr$ with one of the larger class of methods
being constraint-based approaches~\citep{
  margaritis1999,tsamardinos2003,schluter2014,ling2019}. These
approaches iteratively test conditional independence statements---using
data sampled from $\pr$---to grow and shrink a candidate
\gls{mb} of $Y$. However, the \gls{mb} of $Y$ in
the joint distribution $\pr$ might not be the same as its \gls{mb}
in the graph $\gr$. Following the Global Markov Property in
\Cref{def:global-markov}, we at know that the \gls{mb} of
$Y$ in $\pr$ is a subset of the \gls{mb} in $\gr$---which this
might prevent us from fully recovering the \gls{mb} in
$\gr$. Therefore, in order to ensure that the Markov blankets of $Y$ in
both $\pr$ and $\gr$ are equal, a common assumption constraint-based
approaches make is that $\gr$ is \emph{faithful} to the joint
distribution $\pr$ as defined in \Cref{def:faith}~\citep{spirtes2001}.

\begin{definition}[Faithfulness]\label{def:faith}
  Given a \gls{pgm} $\model=(\gr, \pr)$, the graph
  $\gr$ is \emph{faithful} to the joint distribution $\pr$ if for all
  disjoint subsets $\rmX,\rmY,\rmS\subseteq\verts$~\citep{spirtes2001},
  \begin{equation}
    \label{eq:16}
    \rmX \indep{\pr} \rmY \mid \rmS
    \implies \rmX \indep{\gr} \rmY \mid \rmS
  \end{equation}
  which is basically the reverse direction of the global Markov property
  in \Cref{def:global-markov}.
\end{definition}

Therefore, by assuming both the Global Markov Property faithfulness, the
conditional independencies of both $\gr$ and $\pr$ then coincide,
$\indeps{\gr}=\indeps{\pr}$. However faithfulness can be violated in
many ways and we will explore some of them in the following section.

\section{Faithfulness Violations and Existing Relaxations}\label{sec:violation}
Probably the most well-known example of a faithfulness violation is the
(noisy) XOR function $Y\coloneq X_{1}\oplus X_{2}$ with the graphical
structure
$X_{1}\rightarrow Y \leftarrow
X_{2}$~\citep{inazumi2011,marx2021}\footnote{The (noisy) XOR function is
a special case of the (noisy) parity function whose faithfulness
violation we detail in \Cref{ex:parity}.}.
A
recent faithfulness relaxation for tackling XOR-type relations that is
of particular interest to us is $2$-adjacency
faithfulness~\citep{marx2021}.
\begin{definition}[2-adjacency faithfulness]
  Given the \gls{pgm} $\model=(\gr,\pr)$, $\gr$ is $2$-adjacent faithful
  to $\pr$ if for all adjacent variables $(X,Y)\in\edges$, $\exists \rmX
  \subseteq \mbs{\gr}(Y)$ where $X\in\rmX$ s.t. $\forall X \in \rmX: X
  \dep{\pr} Y \mid \rmX \setminus X$ and if $|\rmX|=2$,
  $X_{1}\dep{\pr}X_{2} \mid Y$.
\end{definition}
That said XOR-type relationships only involve three variables. Since we
are interested in higher-order unfaithfulness, we will use the
noisy parity function as a primary example of the type of faithfulness
violation we are interested in.
\begin{figure}
  \centering
  \begin{tikzpicture}
  \node[draw, circle, inner sep=2pt] (1) at (-1,1.85) {$X_{1}$};
  \node[draw, circle, inner sep=2pt] (2) at (0,1.85) {$X_{2}$};
  \node[draw, circle, inner sep=2pt] (3) at (1,1.85) {$X_{3}$};
  \node[draw, circle, inner sep=3pt] (y) at (0,0.6) {$Y$};
  \draw[-{Stealth}] (1) -- (y);
  \draw[-{Stealth}] (2) -- (y);
  \draw[-{Stealth}] (3) -- (y);
\node (PrX1) at (0, -0.75){
    \begin{tabular}[t]{|cc|}
      \hline
      \multicolumn{2}{|c|}{$\forall i \in [3]$}\\
      $\pr_{X_{i}}(0)$ & $\pr_{X_{i}}(1)$\\
      \hline
      0.5 & 0.5\\
      \hline
    \end{tabular}
};

\node (PrY) at (-3.5, 0.5){
    \begin{tabular}[t]{|ccc|c|}
      \hline
      $X_{1}$ & $X_{2}$ & $X_{3}$ & $\pr_{Y}(1)$\\
      \hline
      0 & 0 & 0 & 0.1\\
      0 & 0 & 1 & 0.9\\
      0 & 1 & 0 & 0.9\\
      0 & 1 & 1 & 0.1\\
      1 & 0 & 0 & 0.9\\
      1 & 0 & 1 & 0.1\\
      1 & 1 & 0 & 0.1\\
      1 & 1 & 1 & 0.9\\
      \hline
    \end{tabular}
};
\end{tikzpicture}

  \caption{Example of a noisy parity \gls{bn} over $4$ variables.}
  \label{fig:parity}
\end{figure}
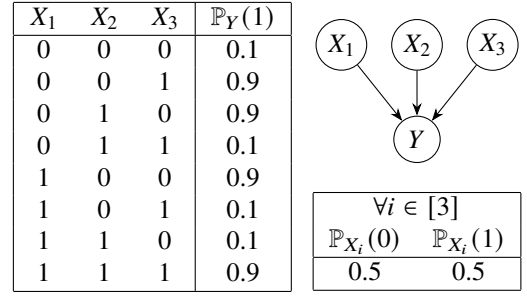
\begin{example}[Noisy Parity Function]\label{ex:parity}
  Let $\model=(\gr,\pr)$ be a \gls{pgm} where
  $\rmX=\{X_{1},X_{2},X_{3}\}$, $\verts=Y\cup\rmX$, and $\gr$ has the
  structure of the \gls{dag} in \Cref{fig:parity}. Each $X\in\rmX$ are
  independent of each other and has the distribution of a fair coin. The
  distribution of $Y$ conditioned on $\rmX$ is then
  $\pr(Y=f(\rvx)\mid\rmX=\rvx) = 0.9$, where,
  $f(\rvx) \equiv \sum_{i=1}^{n}x_{n}\pmod{2}$. For any random variable
  $X\in\rmX$, we know that $\pr(Y,X)=0.25=0.5\times 0.5=\pr(Y)\pr(X)$
  for any value $Y$ and $X$ takes. Therefore it must the case that
  $Y \indep{\pr} X$, but since $Y$ and $X$ are not separable in $\gr$,
  $Y\dep{\gr} X$; $\gr$ is not faithful to $\pr$.  Additionally for any
  distinct random variables $X_{i},X_{j}\in\rmX$,
  $\pr(Y,X_{i},X_{j})=1/8=0.5^{3}=\pr(Y)\pr(X_{i})\pr(X_{j})$ for any
  values of $Y,X_{i},X_{j}$. Therefore we know that
  $Y\indep{\pr}\{X_{i},X_{j}\}$, $X_{i}\indep{\pr}\{Y,X_{j}\}$, and
  $X_{j}\indep{\pr}\{X_{i},Y\}$. Since $Y$ is adjacent to all
  $X\in\rmX$, but there is no $2$-cardinality subsets $\rmX'\subset\rmX$
  that makes $Y\dep{\pr}\rmX'$; $\gr$ is not $2$-adjacent faithful to
  $\pr$ following the definition in \citet[Definition 7]{marx2021}.
\end{example}

Although we now know of a \gls{bn} that violates both faithfulness and
$2$-adjacency faithfulness, these types of relationships between
variables might very well not be common in nature. In fact, it has been
shown that \emph{unfaithful} \glspl{bn} has Lebesgue measure
zero~\citep{meek1995,spirtes2001,boeken2025}. However, empirical
violations of faithfulness due to limited samples can be quite
common~\citep{uhler2013,lemeire2012,boeken2025}. Therefore in
\Cref{ex:sample}, we will present a scenario where, even though the true
\gls{bn} is faithful; due to sampling error, the \gls{dag} of the
\gls{bn} is not faithful to the empirical distribution obtained from the
sample.
  
\begin{example}\label{ex:sample}
  Similar to \Cref{ex:parity}, let $\model=(\gr,\pr)$ be a \gls{pgm}
  with random variables $\rmX$ and $Y$ with the \gls{dag} structure in
  \Cref{fig:sample}. The distribution of $Y$ conditioned on $\rmX$ is
  then
  $\pr\bigl(Y=\bigl[\sum_{X\in\rmX}X=1\bigr] \bigm| \rmX \bigr) = 0.9$,
  where $[\cdot]$ is the Iverson bracket notation for the indicator
  function. In other words, $Y=1$ has a probability of $0.9$ when only
  one of the variables in $\rmX$ is $1$. Otherwise, $Y=0$ has a
  probability of $0.9$.

  The marginal distribution of $Y$ is $\pr(Y=1)=3/8$ since only 3
  value-combinations of $\rmX$ has just a single $1$ among them. We then
  know that
  $\forall X \in \rmX : \pr(Y=1,X=1)=1/8\neq
  (3/8)\times(4/8)=\pr(Y=1)\pr(X=1)$. Therefore, $\forall X\in\rmX : Y
  \dep{\pr} X$ and consequently, $\gr$ \emph{is} faithful to $\pr$.

  However, it is possible for a sample $\db$ of $\pr$ to be pathological
  in the sense that, the empirical distribution from $\db$ induces a
  independence between $Y$ and some $X\in\rmX$. For example, the sample
  in \Cref{fig:sample} causes the marginal distribution of $Y$ to be
  $\pre(Y)=1/2$ and the marginal distribution over $Y$ and $X_{1}$ to be
  $\pre(Y,X_{1})=1/4=(1/2)\times(1/2)=\pre(Y)\pre(X_{1})$ for all
  possible values of $Y$ and $X_{1}$. Therefore under this empirical
  distribution $\pre$, $Y\indep{\pre}X_{1}$---which empirically violates
  the faithfulness assumption.
\end{example}

Fortunately in \Cref{ex:sample}, the \gls{dag} in \Cref{fig:sample} is
still $2$-adjacent faithful to the empirical distribution, therefore the
full \gls{mb} is still recoverable by the algorithm outlined in
\citet{marx2021}. However, this is not always the case and as we will
see in \Cref{sec:exp-syn}, sometimes considering higher-order
relationships between variables can help better overcome empirical
faithfulness violations.

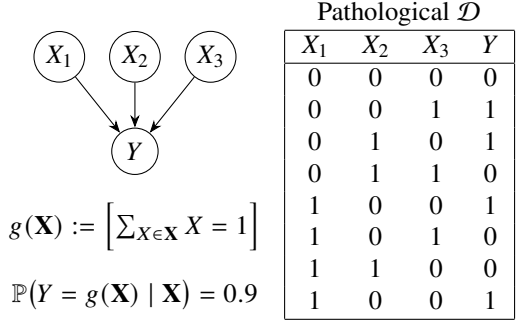
\begin{figure}
  \centering
  \begin{tikzpicture}
  \node[draw, circle, inner sep=2pt] (1) at (-2.5,2.75) {$X_{1}$};
  \node[draw, circle, inner sep=2pt] (2) at (-1.5,2.75) {$X_{2}$};
  \node[draw, circle, inner sep=2pt] (3) at (-0.5,2.75) {$X_{3}$};
  \node[draw, circle, inner sep=3pt] (y) at (-1.5,1.5) {$Y$};
  \draw[-Stealth] (1) -- (y);
  \draw[-Stealth] (2) -- (y);
  \draw[-Stealth] (3) -- (y);
  \node (Z) at (-1.5, 0.5) {$g(\rmX):=\Bigl[\sum_{X\in\rmX}X=1\Bigr]$};
  \node (prY) at (-1.5, -0.4) {$\pr\bigl(Y=g(\rmX)\mid \rmX\bigr)=0.9$};

\node (samp) at (2, 1.4){
    \begin{tabular}[t]{|ccc c|}
      \multicolumn{4}{c}{Pathological $\mathcal{D}$}\\
      \hline
      $X_{1}$ & $X_{2}$ & $X_{3}$ & $Y$\\
      \hline
      0 & 0 & 0 & 0\\
      0 & 0 & 1 & 1\\
      0 & 1 & 0 & 1\\
      0 & 1 & 1 & 0\\
      1 & 0 & 0 & 1\\
      1 & 0 & 1 & 0\\
      1 & 1 & 0 & 0\\
      1 & 0 & 0 & 1\\
      \hline
    \end{tabular}
};


\end{tikzpicture}

  \caption{Example of a noisy exactly-1 \gls{bn} over $4$ variables with
  a pathological sample that causes the \gls{dag} to be unfaithful to
  the empirical distribution.}
  \label{fig:sample}
\end{figure}

\section{The k-Order Faithfulness Relaxation}\label{sec:relax}
In this section we will introduce our generalisation of the faithfulness
assumption in \Cref{def:faith} for finding conditional dependencies
between two variables in $\verts$ that require considering $k$
additional variables to be found. But before that, we first define in
\Cref{def:k-dep} what it means for two distinct variables $Y,X\in\verts$
to only be conditionally dependent given some $k$-cardinality subset
$\rmZ\subseteq\subverts, |\rmZ| = k$.

\begin{definition}[$k$-Order Dependence]\label{def:k-dep}
  Given a \gls{pgm} $\model=(\gr, \pr)$, let $Y,X\in\rmV$ be two
  distinct random variables and $\rmS,\rmZ\subseteq\subverts$ be two
  disjoint subsets. Then we denote the absence of any subset of
  $\rmS'\subseteq\rmS$ that renders $Y$ and $X$
  conditionally independent given $\rmZ \cup \rmS'$ as
  \begin{equation}
    \label{eq:k-dep-asc}
    \kassoc{\rmS}{Y,X}{\rmZ}
    := \forall \rmS' \subseteq \rmS :
    Y \dep{\pr} X \mid \rmZ \cup \rmS'.
  \end{equation}
  Here we say that $Y$ and $X$ are $|\rmZ|$-order associated
  w.r.t. $\rmZ$ over the ``separating'' set $\rmS$.
  Of course under this definition there might be some proper subset
  $\rmZ' \subset \rmZ$ where $\kassoc{\rmS}{Y,X}{\rmZ'}$ still
  holds. Therefore we say $Y$ and $X$ are $|\rmZ|$-order
  \emph{dependent} on $\rmZ$ only if this is not the case,
  \begin{gather}
    \begin{aligned}
      &\kdep{\rmS}{Y,X}{\rmZ}\\
      &:= \kassoc{\rmS}{Y,X}{\rmZ} \wedge
        \forall \rmZ' \subset \rmZ : \neg\,\kassoc{\rmS}{Y,X}{\rmZ'}.
    \end{aligned}
    \label{eq:k-dep-dep}
  \end{gather}
\end{definition}
With this definition of $k$-order dependence, we are
able to capture parity-like relations over $k+2$ variables---such as the
one in \Cref{ex:parity}. Specifically we know that the noisy-parity
relationship in \Cref{ex:parity} is $2$-order dependent since for
distinct variables $A,B\in\{X_{1},X_{2},X_{3},Y\}$, $\rmZ$ must always
contain the other two variables for $\kassoc{\emptyset}{A,B}{\rmZ}$ to
be true.

Further note that when $k=1$, our notion of $k$-order dependence---
similar to \emph{$2$-association} in \citet[Definition 6]{marx2021}---is
able to capture the relationships between \emph{Unfaithful Triples}
which includes XOR-type relationships.

There are two aspects of \Cref{def:k-dep} whose purpose might not
immediately jump out. The first being that in \Cref{eq:k-dep-asc}, we
require $Y\dep{\pr}X\mid\rmZ\cup\rmS'$ to be true for all subsets
$\rmS'\subseteq\rmS$---instead of just
$Y\dep{\pr}X\mid\rmZ\cup\rmS$. This requirement ensures that $Y$ and $X$
are dependent as long as they are conditioned on $\rmZ$---forcing a
distinction between the possible dependants $\rmZ$ and the possible
separators $\rmS$. The other odd aspect of \Cref{def:k-dep} is that we
define both $k$-order association and dependence over some arbitrary
subset $\rmS\subseteq\subverts$ instead of just the entirety of
$\subverts$. This allows us to reason about associations relative to
the candidate Markov blanket maintained by our algorithm, which is
central to proving in \Cref{sec:algo} that {\method} correctly recovers
the Markov boundary of some variable $Y\in\verts$.

Since henceforth we need to reason about subsets with a restricted
cardinality; before continuing we will first define some notational
shorthand for dealing with such sets.

\begin{definition}[Power Sets of Limited
  Cardinality]\label{def:powerset}
  Let $\verts$ be a set of elements and
  $k\in\mathbb{N} : k\leq |\verts|$ be some natural number\footnote{We
    adopt the definition of the natural numbers in which $0$ is
    included.} smaller than the cardinality of $\verts$. Then we use,
  \begin{equation}
    \label{eq:17}
    \powereq{k}{\verts} := \{\rmZ \subseteq \verts : |\rmZ| = k\},
  \end{equation}
  to denote the set of all subsets of $\verts$ with cardinality $k$
  and,
  \begin{equation}
    \label{eq:18}
    \powerleq{k}{\verts} := \{\rmZ \subseteq \verts : |\rmZ| \leq k\},
  \end{equation}
  to denote the set of all the subsets of $\verts$ with cardinality less
  than or equal to $k$.
\end{definition}

Before introducing our generalisation of the faithfulness
assumption to $k$-order dependencies, we first show in
\Cref{lem:k-dep-to-assoc} that $k$-order associations w.r.t. $\rmZ$
implies the existence of a $k$-order dependence w.r.t. some subset of
$\rmZ$.

\begin{restatable}[$k$-Order Dependence Arises from Association]{lemma}{kdeptoassoc}
  \label{lem:k-dep-to-assoc}
  Given a \gls{pgm} $\model=(\gr,\pr)$, let $Y,X\in\rmV$ be
  distinct random variables and $\rmS,\rmZ\subset\subverts$ be disjoint
  subsets. Then the following holds for all $Y,X,\rmS,\rmZ$:
  \begin{equation}
    \label{eq:11}
    \kassoc{\rmS}{Y,X}{\rmZ}
    \implies \exists \rmZ' \subseteq \rmZ :
    \kdep{\rmS}{Y,X}{\rmZ'}.
  \end{equation}
\end{restatable}
The proof of \Cref{lem:k-dep-to-assoc} and all other Lemmas in the paper
can be found in \Cref{apx:proofs}.


\begin{assumption}[$k$-Order Faithfulness]\label{def:k-faith}
  Given the \gls{pgm} $\model=(\gr, \pr)$, $\gr$ is $k$-order faithful
  to $\pr$ if
  $\forall \rmS \subseteq \verts\setminus\{Y,X\}$,
  \begin{gather}
    \label{eq:k-faith-indep}
    \begin{aligned}
      &\forall \rmZ \in \power{\leq k}(\verts\setminus\{Y,X\})\:\:
        \exists \rmS' \subseteq \rmS :
      Y \indep{\pr} X \mid \rmZ \cup \rmS'\\
      &\implies
      Y \indep{\gr} X \mid \rmS.
    \end{aligned}
  \end{gather}
  Its contrapositive then states
  that $\forall \rmS \subseteq \verts\setminus\{Y,X\}$,
  \begin{gather}
    \label{eq:k-faith-1}
    \begin{aligned}
      &Y \dep{\gr} X \mid \rmS\\
      &\implies
      \exists \rmZ \in \power{\leq k}(\verts\setminus\{Y,X\}) :
        \kassoc{\rmS}{Y,X}{\rmZ}.
    \end{aligned}
  \end{gather}
  However, the definition in \Cref{eq:k-faith-1} allows $k$ to be larger
  than it needs to be---i.e. $k$ is not strict. Therefore, we say that
  $\gr$ is \emph{strictly} $k$-order faithful to $\pr$ if
  $\forall \rmS \subseteq \verts\setminus\{Y,X\}$,
  \begin{gather}
    \label{eq:k-faith-2}
    \begin{aligned}
      &Y \dep{\gr} X \mid \rmS\\
      &\implies
      \exists \rmZ \in \powereq{k'}{\verts\setminus\{Y,X\}} :
      \kdep{\rmS}{Y,X}{\rmZ},
    \end{aligned}
  \end{gather}
  where $k'\leq k$. This definition in \Cref{eq:k-faith-2} logically
  follows from \Cref{eq:k-faith-1} as a result of
  \Cref{lem:k-dep-to-assoc}.
\end{assumption}

With this new generalisation of faithfulness, we will now
present an initial algorithm, {\method}, that makes use of our
$k$-order faithfulness assumption for \glspl{mb} discovery.

\section{kOMB: k-Order Markov Blanket}\label{sec:algo}
Our algorithm for \acrfull{mb} discovery, {\method}, is
a modified version of the \acrfull{gs} algorithm by
\citet{margaritis1999}. In \gls{gs}, a candidate \gls{mb} $\rmS$ is
iteratively grown by testing the conditional independence between $Y$
and variables that are not yet in $\rmS$---conditioned the current
candidate \gls{mb} $\rmS$. Therefore, the separating set used in the
conditional independence tests change as the \gls{gs} algorithm
progresses.

In {\method}, something similar occurs where we will only consider
subsets of the candidate \gls{mb} when trying to find subsets that
successfully cause $Y$ and $X$ to be conditionally independent. In other
words, instead of considering the dependence between $Y$ and $X$ over
all the variables in $\verts$, $\kdep{\subverts}{Y,X}{\rmZ}$; we will
instead consider the association between $Y$ and $X$ over just the
current candidate \gls{mb} $\rmS$, $\kassoc{\rmS}{Y,X}{\rmZ}$. As a
first step, we show that a $k$-order dependence between $Y$ and $X$ with
no separating variables to consider implies that $Y$ is $k$-order
associated with every one of the involved variables.

\begin{restatable}[Marginal $k$-Order Dependence Implies Inter-Association]{lemma}{kdepinter}
  \label{lem:k-dep-inter}
  Given a \gls{pgm} $\model=(\gr, \pr)$, let $X,Y\in\verts$ be distinct
  variables and $\rmZ\subseteq\subverts$. Then if $Y$ and $X$ are
  $|\rmZ|$-order dependent w.r.t. $\rmZ$ over the empty separating set,
  $Y$ is $|\rmZ|$-order associated with every variable in $\rmZ\cup X$,
  \begin{gather}
    \kdep{\emptyset}{Y,X}{\rmZ}
    \implies
    \kinter{\emptyset}{Y,\rmZ \cup X},\label{eq:k-dep-inter-2}
    \intertext{where}
    \begin{aligned}
      \kinter{\rmS}{Y,\rmW} := \forall W\in\rmW :
      \kassoc{\rmS}{Y,W}{\rmW\setminus W}.
    \end{aligned}
  \end{gather}
  When \Cref{eq:k-dep-inter-2} is true, we say that $Y$ is
  $|\rmZ|$-order \emph{inter-associated} with $X\cup\rmZ$.
\end{restatable}

\Cref{lem:k-dep-inter} requires no faithfulness assumption---it holds
for every distribution via the semi-graphoid axioms---and generalises
the $2$-association property of unfaithful triples in
\citet[Definition 6]{marx2021} to arbitrary orders; the noisy-parity
\Cref{ex:parity} is exactly this case with $\rmZ=\{X_{2},X_{3}\}$. It
motivates the \textit{make\_strict} routine of \Cref{alg:assoc-mine},
which prunes a found dependence set towards an inter-associated one.

The last consideration we need to make for {\method} to be remotely
feasible---especially in cases where the true \gls{mb} of $Y\in\verts$
in $\gr$ is large---is that, we need to somehow limit the size of the
conditioning set when conducting the \acrfull{ci} tests in
{\method}. When testing if $Y$ and $X$ are $k$-order associated, we will
generally use \gls{ci} tests of the form,
$\forall \rmS' \subseteq \rmS: Y \indep{\pr} X \mid \rmZ \cup \rmS'$,
where $\rmZ$ is the ``dependence'' set that renders $Y$ and $X$
conditionally dependent.  Fortunately, assuming $k$-order faithfulness,
we know that the size of $\rmZ$ is bounded by $k$.

However the same cannot be said for the set $\rmS'$ as it can be as
large as the candidate \gls{mb} $\rmS$--which can be arbitrarily
large depending on $\gr$. Therefore, the third and last assumption we
will make is that if $Y$ and $X$ are conditionally independent given
$\rmZ\cup\rmS$, then there must exist some subset with size less than
$l$, $\rmS'\in\powerleq{l}{\rmS}$, such that $Y$ and $X$ are still
conditionally independent given $\rmZ\cup\rmS'$. 

\begin{assumption}[$l$-Bounded Separator
  Assumption]\label{def:bound-sep}
  For a \gls{pgm} $\model=(\gr,\pr)$, distinct variables $Y,X\in\verts$,
  and $\verts'\coloneq \subverts$, let us first recall the $k$-order
  faithfulness assumption as defined in \Cref{eq:k-faith-indep}
  $\forall\rmS \subseteq \subverts$,
  \begin{gather}
    \tag{\ref{eq:k-faith-indep}}
    \begin{aligned}
      &\forall \rmZ \in \power{\leq k}(\verts')\:\:
        \exists \rmS' \subseteq \rmS :
      Y \indep{\pr} X \mid \rmZ \cup \rmS'\\
      &\implies
      Y \indep{\gr} X \mid \rmS.
    \end{aligned}
  \end{gather}
  Then, for some natural number $l\in\mathbb{N}$, we define an
  $l$-bounded separator version of the $k$-order faithfulness assumption
  $\forall\rmS \subseteq \subverts$,
  \begin{gather}
    \begin{aligned}
      &\forall \rmZ \in \power{\leq k}(\verts')\:\:
        \exists \rmS' \in \powerleq{l}{\rmS}:
        Y \indep{\pr} X \mid \rmZ \cup \rmS'\\
      &\iff
        \forall \rmZ \in \power{\leq k}(\verts')\:\:
        \exists \rmS' \subseteq \rmS :
      Y \indep{\pr} X \mid \rmZ \cup \rmS'\\
      &\implies
      Y \indep{\gr} X \mid \rmS,
    \end{aligned}
  \end{gather}
  which essentially states that as long as $Y$ and $X$ are conditionally
  independent given $\rmZ$ and every $\leq l$-cardinality subset of
  $\rmS$, then $Y$ and $X$ are conditionally independent given $\rmZ$
  and every possible subset of $\rmS$---regardless of cardinality.

  The contrapositive of this assumption is then
\begin{gather}
  \begin{aligned}
    &Y \dep{\gr} X \mid \rmS\\
    &\implies
      \exists \rmZ \in \power{\leq k}(\verts')\:\:
      \forall \rmS' \subseteq \rmS:
      Y \dep{\pr} X \mid \rmZ \cup \rmS'\\
    &\iff
      \exists \rmZ \in \power{\leq k}(\verts') :
      \kassoc{\rmS}{Y,X}{\rmZ}\\
    &\iff \exists \rmZ \in \powerleq{k}{\verts'} :
    \lassoc{\rmS}{Y,X}{\rmZ},
  \end{aligned}
\end{gather}
where
$\lassoc{\rmS}{Y,X}{\rmZ} \coloneq \forall \rmS' \in
\powerleq{l}{\rmS}: Y \dep{\pr} X \mid \rmZ \cup \rmS'$.
\end{assumption}

The concept of bounding the maximum size of the separators in a graph
has been explored before. For instance, \citet{soh2019} defined an
undirected graph $\gr$ to be weakly $K$-separable if for any two
non-adjacent vertices in $\gr$, there exists a set of vertices with
cardinality $\leq K$ that separates the two vertices in
$\gr$. However, instead of focusing on separation in the graph $\gr$,
our $l$-bounded separator assumption focuses on the cardinality of
``separating'' sets that causes conditional independencies between
variables in $\verts$.

The $l$-bounded separator assumption is also analogous to assumptions
that existing constraint-based \gls{mb} discovery methods make about the
size of the conditioning sets needed to decide conditional independence
reliably from finite samples~\citep{
  margaritis1999,tsamardinos2003,aliferis2003,tsamardinos2003b}.
\Cref{def:bound-sep} simply makes such a bound explicit.  Furthermore,
as shown later in \Cref{apx:dsep}, the bound is mild on the networks
used in \Cref{sec:exp-real}.

\subsection{Theoretical Algorithm}
\begin{algorithm}[t]
  \caption{kOMB}\label{alg:komb}
  \SetKwProg{Fn}{Function}{:}{}
  \KwIn{$\db$, $\model$, $Y$, $\alpha$, $k$, $l$}
  \KwOut{$\mbs{\gr}(Y)$}
  $\rmS \gets \emptyset$\tcp*[r]{Candidate Markov Blanket}
  \Repeat
  {$\rmX' = \emptyset$}{
    $\verts' \gets \verts \setminus (Y \cup \rmS)$\;
    $\rmX' \gets $ find\_inter\_dep($Y$, $\emptyset$, $\rmS$, $\verts'$, $k$, $l$)\;
    $\rmS \gets \rmS \cup \rmX'$
  }
  \While{$\exists X \in \rmS : \textit{find\_cond}(Y, X,
    \emptyset, \rmS\setminus X, k, l) = \bot$}{
    $\rmS \gets \rmS \setminus X$\;
  }
  \Return{$\rmS$}\;
\end{algorithm}

As a proof-of-concept, in this section we propose {\method}---a
modification of the \acrfull{gs} by \citet{margaritis1999} to discover
Markov blankets with higher-order relationships under the following
assumptions: 1) the Global Markov Property from
\Cref{def:global-markov}, 2) $k$-order faithfulness as in
\Cref{def:k-faith}, and 3) the $l$-bounded separator assumption in
\Cref{def:bound-sep}.

The original \gls{gs} algorithm starts with an empty \gls{mb}
$\rmS=\emptyset$ and iteratively adds a variable $X$ into $\rmS$ if the
\acrfull{ci} test $Y \indep{\pr} X \mid \rmS$ indicates that $Y$ is
conditionally independent to $X$ given $\rmS$. The main difference
between {\method} and the \gls{gs} algorithm then is: 1) considering the
addition of entire sets of variables
$X \cup \rmX : \rmX \in \powerleq{k}{\verts\setminus Y}, X \in
\verts\setminus(Y\cup\rmS\cup\rmX)$ at each iteration to
find higher-order parity-type relationships between variables, and 2)
conducting \gls{ci} tests using just subsets of the current candidate
\gls{mb} $\rmS$ in the condition to ensure that these \gls{ci}
tests remain feasible even when the cardinality of $\rmS$ is large.
As such, we break up {\method} into 3 sub-algorithms:
\begin{description}
\item[\Cref{alg:komb}] is just the entry-point function that implements
  the overall logic of the \gls{gs} algorithm, with the main
  modifications and additions located in the other two algorithms,
\item[\Cref{alg:assoc-mine}] is overall responsible for finding some
  subset of the variables not already in the candidate \gls{mb},
  $\rmX\in\powerleq{k}{\verts\setminus(X\cup Y\cup \rmS)}$, that might
  cause $Y$ and $X$ to be conditionally dependent given $\rmX$ and some
  subset $\rmX'\in\powerleq{k-|\rmX|}{\rmS}$ of the current candidate
  \gls{mb}, finally,
\item[\Cref{alg:ci-test}] is then tasked with finding such a
  $\rmX'\in\powerleq{k-|\rmX|}{\rmS}$ that causes $Y \dep{\pr} Y \mid
  \rmX \cup \rmX'$ and checking if there are any $l$-bounded subsets of
  the current candidate \gls{mb}, $\rmS'\in\powerleq{l}{\rmS}$,
  such that $Y \indep{\pr} Y \mid \rmX \cup \rmX' \cup \rmS'$.
\end{description}

Keeping this outline in mind, we now present the full
algorithm for {\method} and show that it correctly discovers the \gls{mb} of some $Y\in\verts$ assuming that 
\Cref{def:global-markov,def:k-faith,def:bound-sep} hold. The proofs for
any subsequent Theorems can be found in \Cref{apx:proofs-thm}.

\begin{algorithm}[t]
  \caption{Association Mining}\label{alg:assoc-mine}
  \SetKwProg{Fn}{Function}{:}{}
  \Fn{make\_strict($Y$, $X$, $\rmZ$, $\rmS$, $k$, $l$)}{
    \For{$Z \in \rmZ$}{
      \If{$\neg$ no\_seps($Y$, $Z$, $X\cup\rmZ
        \setminus Z$, $\rmS$, $l$)}{
        remove $Z$ from $\rmZ$\;
      }
    }
    \Return{$X\cup\rmZ$}\;
  }
  \Fn{find\_inter\_dep($Y$, $\rmX$, $\rmS$, $\rmV'$, $k$, $l$)}{
    \If{$|\rmX| > k$}{
      \Return{$\emptyset$}\;
    }
    \For{$X \in \rmV\setminus(Y\cup\rmS \cup \rmX)$}{
      $\rmZ\gets$\textit{find\_cond}
      ($Y$, $X$, $\rmX$, $\rmS$, $k$, $l$)\;
      \If{$\rmZ \neq \bot$}{
        \Return{make\_strict($Y$, $X$, $\rmZ$, $\rmS$, $k$, $l$)}\;
      }
    }
    \For{$X\in\rmV'\setminus(Y\cup\rmS\cup\rmX)$}{
        $\rmV'\gets \rmV' \setminus X$\;
        $\rmX' \gets $ find\_inter\_dep(
        $Y$, $\rmX\cup X$, $\rmS$, $\rmV'$, $k$, $l$)\;
        \If{$\rmX' \neq \emptyset$}{
          \Return{$\rmX'$}\;
        }        
    }
    \Return{$\emptyset$}\;
  }
\end{algorithm}

\begin{restatable}[Correctness of \Cref{alg:ci-test}]{theorem}{citest}
  \label{thm:citest}
  Given an estimated distribution $\pre$, distinct variables
  $Y,X\in\verts$, disjoint subsets $\rmX,\rmS\subset\subverts$, and the
  natural numbers $k,l\in\mathbb{N}$; the function
  \textit{find\_cond} in \Cref{alg:ci-test} returns the set
  $\rmX'\cup\rmX$ if and only if
  $\exists \rmX' \in \powerleq{k-|\rmX|}{\rmS}$ where,
  \begin{equation}
    \label{eq:25}
    \forall \rmS' \in \powerleq{l}{\rmS \setminus \rmX'} :
    Y \dep{\pre} X \mid \rmX \cup \rmX' \cup \rmS'.
  \end{equation}
  Otherwise \textit{find\_cond} returns the failure value $\bot$,
  distinct from the empty set, which it may return on success when
  $\rmX'=\rmX=\emptyset$.
\end{restatable}

\begin{restatable}[Correctness of \Cref{alg:assoc-mine}]{theorem}{assocmine}
  \label{thm:assocmine}
  Given an estimated distribution $\pre$, the target variable
  $Y\in\verts$, the current candidate \gls{mb}
  $\rmS\subset\verts\setminus Y$, the remaining variables
  $\verts'=\verts\setminus(Y\cup \rmS)$, and the natural numbers
  $k,l\in\mathbb{N}$; the call
  \textit{find\_inter\_dep}$(Y,\emptyset,\rmS,\verts',k,l)$ returns the
  empty set if and only if
  \begin{equation}
    \label{eq:fid-empty}
    \begin{aligned}
      &\forall X \in \verts'\:\:
        \forall \rmZ \in \powerleq{k}{\verts\setminus\{Y,X\}}\:\:
        \exists \rmS' \in \powerleq{l}{\rmS\setminus\rmZ} :\\
      &\qquad Y \indep{\pre} X \mid \rmZ \cup \rmS'.
    \end{aligned}
  \end{equation}
  Otherwise it returns a nonempty set $X\cup\rmZ_{R}$ with
  $X\in\verts'$ and $\rmZ_{R}\subseteq\rmZ$ for some
  $\rmZ\in\powerleq{k}{\verts\setminus\{Y,X\}}$ satisfying the estimated
  $l$-bounded association
  $\forall\rmS'\in\powerleq{l}{\rmS\setminus\rmZ}:
  Y\dep{\pre}X\mid\rmZ\cup\rmS'$.
\end{restatable}

\begin{algorithm}[t]
  \caption{Practical Conditional Independence Test}\label{alg:ci-test}
  \SetKwProg{Fn}{Function}{:}{}
  \Fn{no\_seps($Y$, $X$, $\rmZ$, $\rmS$, $l$)}{
    \For{$\rmS' \in \powerleq{l}{\rmS \setminus \rmZ}$}{
      \If{$Y \indep{} X \mid \rmZ \cup \rmS'$}{
        \Return{False}\;
      }
    }
    \Return{True}\;
  }
  \Fn{find\_cond($Y$, $X$, $\rmX$, $\rmS$, $k$, $l$)}{
    \For{$\rmX' \in \powerleq{k-|\rmX|}{\rmS}$}{
      $\rmZ \gets \rmX' \cup \rmX$\;
      \If{$Y \dep{} X \mid \rmZ$}{
        \If{no\_seps($Y$, $X$, $\rmZ$, $\rmS$, $l$)}{
          \Return{$\rmZ$}\;
        }
      }
    }
    \Return{$\bot$}\;
  }
\end{algorithm}

\begin{restatable}[Correctness of \Cref{alg:komb}]{theorem}{algkomb}
  \label{thm:algkomb}
  Given an estimated distribution $\pre$ over variables $\verts$, the
  target variable whose graphical Markov blanket we want to discover
  $Y\in\verts$,
  and the natural numbers $k,l\in\mathbb{N}$; the grow phase of
  {\method} then grows the candidate \gls{mb} $\rmS$ such that,
  \begin{equation}
    \label{eq:27}
    \mbs{\gr}(Y) \subseteq \rmS.
  \end{equation}
  If, in addition, every true blanket member admits an in-blanket
  dependence witness---i.e. for every $X\in\mbs{\gr}(Y)$ and every
  $\rmS$ with $\mbs{\gr}(Y)\subseteq\rmS\subseteq\verts\setminus Y$
  there is a witness $\rmZ\in\powerleq{k}{\mbs{\gr}(Y)\setminus X}$ with
  $\lassoc{\rmS\setminus(X\cup\rmZ)}{Y,X}{\rmZ}$---then the shrink phase
  of {\method} removes variables from $\rmS$ until,
  \begin{equation}
    \label{eq:28}
    \mbs{\gr}(Y) = \rmS,
  \end{equation}
  where it subsequently terminates and returns $\rmS$.
\end{restatable}

Beyond \Cref{def:global-markov,def:k-faith,def:bound-sep}, the recovery
guarantee of \Cref{thm:algkomb} assumes the \gls{ci} decisions are
\emph{sufficiently accurate}---i.e.\ they reflect the true conditional
independencies of $\pr$.

\subsection{Computational Complexity of {\method}}\label{sec:complexity}
As {\method} is a proof-of-concept, we prioritised correctness and
clarity over efficiency. It is nonetheless instructive to characterise
its worst-case time complexity, which we state in
\Cref{cor:complexity} with proof given in \Cref{apx:proof-complex}.

\begin{restatable}[Computational Complexity of {\method}]{corollary}{kombcomplexity}
  \label{cor:complexity}
  Let $n=|\verts|$ denote the number of variables, $N$ the sample size,
  and $K$ the maximum size of any candidate \gls{mb} $\rmS$ encountered
  by \Cref{alg:komb} during execution. Furthermore assume each \gls{ci}
  test runs in $O(N)$ time. Then, for a given order $k$ and separator
  bound $l$, {\method} (\Cref{alg:komb}) terminates in worst-case time
  \begin{equation}
    \label{eq:complexity}
    O\!\left(
      K\, n \sum_{i=0}^{k}\binom{n}{i}
      \left[
        \sum_{i'=0}^{k-i}\binom{K}{i'}
        \sum_{j=0}^{l}\binom{K-i'}{j}\, N
      \right]
    \right),
  \end{equation}
  where we adopt the convention that $\binom{n}{m}=0$ whenever $n<m$.
\end{restatable}

Treating $k$ and $l$ as constants, \Cref{eq:complexity} is bounded by
$O(n^{k+1}\,K^{k+l+1}\,N)$---polynomial in $n$, $K$, and $N$, and for
$k=0$ it reduces to a scan that is linear in $n$, recovering the
original \gls{gs} algorithm. Crucially, the search for a
dependence-exposing subset and its separator is governed by $K$ rather
than $n$, a direct consequence of the $l$-bounded separator assumption
(\Cref{def:bound-sep}). The remaining exponential dependence on $k$ and
$l$ motivates keeping them small. The ablation study in
\Cref{apx:ablation} showcases how differing values of $k$ and $l$
affects \Cref{alg:komb} runtime and its ability to find high-order
dependencies. Generally, from the results in \Cref{apx:ablation} for
the benchmark datasets used, we found that setting $k=l$ and having
$k\leq 2$ provides a decent trade-off between runtime and finding
high-order dependencies.

\begin{table}
  \centering
  \caption{Results on Synthetically Generated Data}
  \label{tab:syn}
\setlength{\tabcolsep}{5pt}
\begin{tabular}{lccccc}
\toprule
Method & and & ex-1 & ex-2 & or & parity \\
\midrule


GS
& 0.410 & 0.313 & 0.288 & 0.365 & 0.025 \\
& {\scriptsize $\pm$0.405} & {\scriptsize $\pm$0.419} & {\scriptsize $\pm$0.407} & {\scriptsize $\pm$0.374} & {\scriptsize $\pm$0.110} \\







kOMB 
& 0.410 & 0.283 & 0.268 & 0.350 & 0.025 \\
{\scriptsize $(k=0,l=3)$}
& {\scriptsize $\pm$0.405} & {\scriptsize $\pm$0.372} & {\scriptsize $\pm$0.374} & {\scriptsize $\pm$0.351} & {\scriptsize $\pm$0.110} \\

kOMB
& \textbf{0.618} & 0.675 & 0.675 & 0.560 & 0.050 \\
{\scriptsize $(k=1,l=3)$}
& {\scriptsize $\pm$0.374} & {\scriptsize $\pm$0.474} & {\scriptsize $\pm$0.474} & {\scriptsize $\pm$0.435} & {\scriptsize $\pm$0.221} \\

kOMB
& \textbf{0.618} & \textbf{1.000} & \textbf{1.000} & \textbf{0.710} & \textbf{1.000} \\
{\scriptsize $(k=2,l=3)$}
& {\scriptsize $\pm$0.374} & {\scriptsize $\pm$0.000} & {\scriptsize $\pm$0.000} & {\scriptsize $\pm$0.384} & {\scriptsize $\pm$0.000} \\

\bottomrule
\end{tabular}

\end{table}

\section{Experiments}\label{sec:exp}
In order to empirically test {\method} we implemented the algorithm as a
\texttt{Python} package written in \texttt{Rust}. The implementation
uses the G-Test with $\alpha=0.01$ as its conditional independence test
of choice when possible and falls back to the SCI test~\citep{marx2019}
when G-test is too weak for the test being conducted---i.e. when the
average frequency for each cell is $<5$. A public repo for our
implementation of {\method} and code for the experiments can be found at
\url{https://github.com/lklee9/k-order-Markov-blanket}. All experiments
were run on an Apple M4 Mac~mini.

Throughout this section we will compare {\method} against existing
methods for \acrfull{mb} discovery. Specifically we will compare
against the following eight \gls{mb} discovery methods:
GS~\citep{margaritis1999}, IAMB~\citep{tsamardinos2003},
HITON~\citep{aliferis2003}, MMMB~\citep{tsamardinos2003b},
PCMB~\citep{pena2007}, LRH~\citep{liu2016}, STMB~\citep{gao2017a},
BAMB~\citep{ling2019}. We will use the implementation of these methods
found in the \texttt{pyCausalFS} python library~\citep{yu2020a} with a
significance level of $0.01$ for any statistical tests as well.

When comparing between different \gls{mb} discovery methods, we
will use the F1 score defined as: $F1 = 2 \times \mathit{precision}
\times \mathit{recall} / (\mathit{precision}+\mathit{recall})$, which is
the harmonic mean of \emph{precision} and \emph{recall}. Here
\emph{precision} denotes the ratio between the number of true positives
in the inferred \gls{mb} and the size of the inferred \gls{mb}. On the
other hand, \emph{Recall} denotes the ratio between the number of true
positives in the inferred \gls{mb} and the size of the true
\gls{mb}. Therefore $F1=1$ represents the perfect precision and recall.

\begin{figure}[t]
  \centering
  \includegraphics[width=\columnwidth]{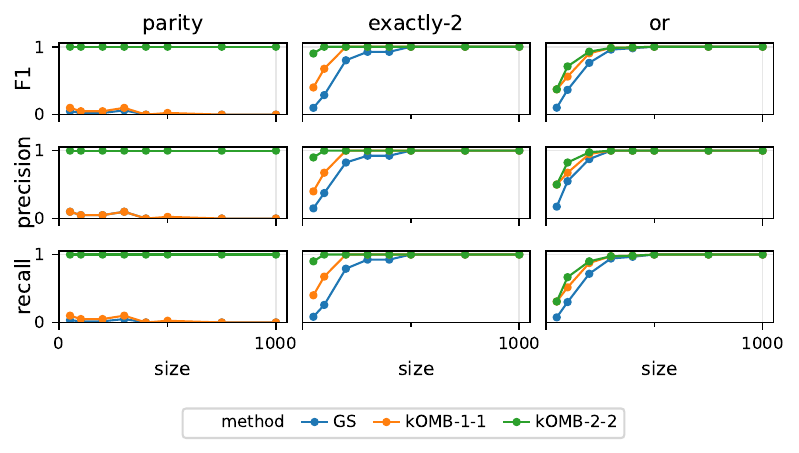}
  \caption{Precision, recall, and F1 score of \gls{gs} and
    ({\method}-$k$-$l$) when recovering the \glspl{mb} of the synthetic
    \glspl{bn} in \Cref{sec:exp-syn}, over sample size. Each row
    corresponds to a metric and each column to a \gls{bn}.}
  \label{fig:syn-metrics}
\end{figure}

\begin{table*}[t]
  \centering
  \caption{F1 score and runtime (in seconds) on benchmark datasets. Best F1
    and fastest runtime per dataset in bold.}
  \label{tab:real}
  \setlength{\tabcolsep}{4pt}
\begin{tabular}{lcccccccc}
\toprule
 & \multicolumn{4}{c}{F1 Score} & \multicolumn{4}{c}{Runtime (s)} \\
\cmidrule(lr){2-5}\cmidrule(lr){6-9}
Method & Alarm1 & Barley & Insurance & Mildew & Alarm1 & Barley & Insurance & Mildew \\
\midrule

BAMB
& 0.6742 & 0.3385 & 0.6588 & 0.4945 & 3.517 & 1.149 & 5.615 & 0.797 \\
& {\scriptsize $\pm$0.1771} & {\scriptsize $\pm$0.1226} & {\scriptsize $\pm$0.0797} & {\scriptsize $\pm$0.1084} & {\scriptsize $\pm$4.305} & {\scriptsize $\pm$0.681} & {\scriptsize $\pm$5.358} & {\scriptsize $\pm$0.392} \\

GS
& 0.3434 & 0.2113 & 0.4615 & 0.2867 & 4.875 & 1.575 & 4.433 & 1.035 \\
& {\scriptsize $\pm$0.1503} & {\scriptsize $\pm$0.0807} & {\scriptsize $\pm$0.1477} & {\scriptsize $\pm$0.1581} & {\scriptsize $\pm$5.183} & {\scriptsize $\pm$0.606} & {\scriptsize $\pm$3.792} & {\scriptsize $\pm$0.341} \\

HITON\_MB
& 0.7636 & 0.3400 & 0.6746 & 0.4918 & \textbf{2.560} & \textbf{1.027} & 3.514 & 0.604 \\
& {\scriptsize $\pm$0.1861} & {\scriptsize $\pm$0.1284} & {\scriptsize $\pm$0.1594} & {\scriptsize $\pm$0.1487} & {\scriptsize $\pm$1.277} & {\scriptsize $\pm$0.298} & {\scriptsize $\pm$2.525} & {\scriptsize $\pm$0.218} \\

MMMB
& 0.7692 & 0.3298 & 0.6978 & 0.4861 & 2.724 & 1.090 & \textbf{3.508} & \textbf{0.553} \\
& {\scriptsize $\pm$0.1874} & {\scriptsize $\pm$0.1274} & {\scriptsize $\pm$0.1405} & {\scriptsize $\pm$0.1810} & {\scriptsize $\pm$1.278} & {\scriptsize $\pm$0.374} & {\scriptsize $\pm$2.525} & {\scriptsize $\pm$0.128} \\

kOMB
& 0.7804 & \textbf{0.5158} & 0.7136 & \textbf{0.6220} & 5.258 & 58.190 & 4.856 & 17.698 \\
{\scriptsize $(k=1,l=1)$} & {\scriptsize $\pm$0.1145} & {\scriptsize $\pm$0.1964} & {\scriptsize $\pm$0.1328} & {\scriptsize $\pm$0.1595} & {\scriptsize $\pm$2.921} & {\scriptsize $\pm$31.619} & {\scriptsize $\pm$2.041} & {\scriptsize $\pm$10.205} \\

kOMB
& \textbf{0.8209} & 0.3128 & \textbf{0.7336} & 0.4360 & 18.144 & 335.821 & 14.493 & 121.909 \\
{\scriptsize $(k=2,l=2)$} & {\scriptsize $\pm$0.1914} & {\scriptsize $\pm$0.1582} & {\scriptsize $\pm$0.1624} & {\scriptsize $\pm$0.1394} & {\scriptsize $\pm$7.566} & {\scriptsize $\pm$135.966} & {\scriptsize $\pm$7.280} & {\scriptsize $\pm$33.153} \\
\bottomrule
\end{tabular}

\end{table*}

\subsection{Experiments with Synthetic Bayesian
  Networks}\label{sec:exp-syn}
To systematically test how {\method} behaves in the face of true and
empirical faithfulness violations, we will task {\method} at learning
the \gls{mb} of every variable for five different basic toy Bayesian
networks. These problems will have $4$
variables---$\rmX=\{X_{1},X_{2},X_{3}\}$ and $Y$---with a \gls{dag} that
has edges $\forall i\in[3] : X_{i}\rightarrow Y$. Each variable
$X\in\rmX$ is an independent Bernoulli random variable with $p=0.5$. The
conditional distribution of $Y$ is then $\pr(Y=f(\rmX)\mid\rmX)=0.9$
where $f(\rmX)$ is one of $5$ different Boolean functions, one for each
toy problem which we describe in \Cref{apx:bool-func}.

We then sample $100$ samples from the $5$ Bayesian networks $10$
different times, leading to $10\times 5$ samples of size $100$---hence
the size of each sample is much larger than the cardinality of the
domain over all variables, $100 >> 2^{4}$. We then tasked each method to
infer---for each Bayesian network---the \gls{mb} of all $4$ variables
from the $10$ different samples separately. We then averaged the F1
score over the $10\times 4$ runs and variables for each Bayesian
network, resulting in \Cref{tab:syn}. Note that we only include the
results for {\method} and \acrfull{gs} because \gls{gs} was the best
performing method out of all the previous methods mentioned in
\Cref{sec:exp} for these synthetic datasets.

From the results in \Cref{tab:syn}, we can observe that {\method}
consistently outperforms the other approaches at \gls{mb}
discovery. More interestingly---even for the problems whose Bayesian
networks are faithful---as we increase $k$ and look at higher order
relationships, {\method} achieves a better F1 score and therefore is
capable of better overcoming empirical unfaithfulness. Furthermore the
\gls{mb} for the parity problem is only recoverable when $k=2$, which
confirms {\method}'s ability to find these higher-order parity-like
relations.

In order to better understand the source of these empirical
unfaithfulness in the low sample size regime, \Cref{fig:syn-metrics}
traces the precision, recall, and F1 score as the sample size grows from
$50$ to $100$. As expected only {\method}-2-2 is capable of recovering
the higher-order dependence of the parity network regardless of sample
size. From the other faithful networks, we can observe that at small
sample size, all methods exhibit lower recall and precision which
indicates dependencies missed by the \gls{mb} discovery methods and
spurious dependencies erroneously detected by said methods respectively.
That said, we can also observe that for {\method}, the higher the value
of $k$ used, the greater its performance in the low sample regime and
the quicker it recovers the true \glspl{mb} for all variables.

\subsection{Experiments with Benchmark Bayesian
  Networks}\label{sec:exp-real}

To further explore how {\method} behaves in more realistic scenarios, we
then repeated the same procedure as \Cref{sec:exp-syn}, but with some
well-known benchmark datasets\footnote{The datasets can be accessed at:
  \url{https://pages.mtu.edu/~lebrown/supplements/mmhc_paper/mmhc_index.html}}
in the \gls{mb} discovery literature. The main difference in the
experiments in this section is that we used samples of $5000$ and we
only tasked each method at finding the \gls{mb} for the $10$ variables
with the largest neighborhoods in the Bayesian network.  Furthermore, we
also omit some of the poorer performing methods in \Cref{tab:real} to
conserve space.

From the results in \Cref{tab:real} we can observe that {\method} does
generally outperform the other existing approaches to \gls{mb}
discovery. Whether a higher order helps, however, depends on the network's
cardinality. On the lower-cardinality Alarm1 and Insurance networks
{\method} with order $2$ improves on order $1$, whereas on the
high-cardinality Barley and Mildew networks order $1$ is best and order
$2$ is substantially worse---there the larger conditioning sets leave the
higher-order \gls{ci} tests underpowered at this sample size. Runtime, in
turn, grows with both the order $k$ and the cardinality of the
network. Specifically, the baselines finish within roughly $6$\,s per
dataset, while {\method} ranges from comparable at order $1$ on the
lower-cardinality networks (Alarm1, Insurance) to a few hundred times
slower at order $2$ on the high-cardinality networks (Barley,
Mildew). The full results and runtimes for every method and every
$(k,l)$ configuration are given in \Cref{tab:real-full,tab:runtime-full}
under \Cref{apx:ablation}.



\section{Conclusion}
We introduced a $k$-order relaxation of faithfulness that explicitly
accounts for parity-type dependencies among $k{+}2$ variables, and used
it to derive a proof-of-concept \acrfull{mb} discovery algorithm,
{\method}; with theoretical guarantees that {\method} recovers the 
graphical \gls{mb} under the proposed assumptions. 
Empirically, {\method} overcomes both true and empirical
faithfulness violations on synthetic problems, and it performs
competitively on commonly used benchmark datasets.

As {\method} is a proof-of-concept demonstrating that exploiting
higher-order dependencies can improve \gls{mb} discovery, an important
direction for future work is the development of \emph{approximate}
algorithms that uncover these higher-order dependencies without
{\method}'s worst-case cost---for example, through more intricate
pruning rules that make {\method} scale better with larger values of
$k$. Secondly, although our ablation (\Cref{apx:ablation}) shows that
$l=k$ with $k\leq 2$ is a robust default, principled \emph{a priori}
selection of $k$ and $l$ without expert knowledge remains open. A
practical alternative is to run {\method} for several values of $k$ and
$l$ while memoising results across runs to speed up the total
runtime. Finally, our evaluation compares against established
constraint-based \gls{mb} discovery methods using the G-test (with an
SCI fallback) for conditional independence. Therefore, utilising more
recent alternative \gls{ci} tests is a further avenue for future work.

\newpage



\begin{acknowledgements} 
  This research has been funded by the Federal Ministry of Research,
  Technology and Space of Germany and the state of North
  Rhine-Westphalia as part of the Lamarr Institute for Machine Learning
  and Artificial Intelligence.
\end{acknowledgements}

\bibliography{ref}

\newpage


\title{High-Order Markov Blanket Discovery\\
  via a k-Order Relaxation of the Faithfulness Assumption
  \\(Supplementary Material)}
\maketitle
\appendix

\section{Graphoid Axioms}\label{apx:graphoid}
\begin{definition}[Graphoid Axioms]\label{def:graphoid}
  For disjoint subsets $\rmX,\rmY,\rmZ,\rmW\subseteq\verts$, a
  semi-graphoid independence model $\indeps{\cdot}$ obeys the following
  properties~\citep{dawid1979,pearl1987}:
  \begin{enumerate}
  \item Symmetry:
    $\rmX \indep{} \rmY \mid \rmZ \implies \rmY\indep{}\rmX\mid\rmZ$
  \item Decomposition:
    $\rmX \indep{} \rmY \cup \rmW \mid \rmZ \implies
    \rmX\indep{}\rmY\mid\rmZ$
  \item Weak Union:
    $\rmX\indep{}\rmY\cup\rmW\mid\rmZ \implies \rmX\indep{}\rmY \mid
    \rmW\cup\rmZ$
  \item Contraction:
    $(\rmX\indep{}\rmY \mid \rmW\cup\rmZ) \wedge
    (\rmX\indep{}\rmW\mid\rmZ)\implies \rmX\indep{}\rmY\cup\rmW\mid\rmZ$
  \end{enumerate}
  The independence model is called a graphoid if it also obeys the
  following additional property~\citep{pearl1985}:
  \begin{enumerate}[resume]
  \item Intersection:
    $(\rmX \indep{} \rmY \mid \rmZ \cup \rmW) \wedge
    (\rmX \indep{} \rmW \mid \rmZ \cup \rmY)
    \implies \rmX \indep{} \rmY \cup \rmW \mid
    \rmZ$
  \end{enumerate}
  All graph independence models $\indeps{\gr}$ are guaranteed to
  obey all the graphoid axioms~\citep{lauritzen2018}. However,
  probabilistic independence models $\indeps{\pr}$ are only guaranteed to
  obey the semi-graphoid axioms. However if the joint distribution is
  strictly positive, then $\indeps{\pr}$ is guaranteed to be a
  graphoid as well~\citep{pearl1988,pearl1989}.
\end{definition}

\section{Proofs}\label{apx:proofs}
\kdeptoassoc*
\begin{proof}
  Consider the family of association witnesses contained in $\rmZ$,
  \begin{equation*}
    \mathcal{W} \coloneq
    \{\rmW \subseteq \rmZ : \kassoc{\rmS}{Y,X}{\rmW}\}.
  \end{equation*}
  By the hypothesis $\kassoc{\rmS}{Y,X}{\rmZ}$ we have $\rmZ\in\mathcal{W}$,
  so $\mathcal{W}$ is nonempty; and since $\rmZ$ is finite, so is
  $\mathcal{W}$. Hence $\mathcal{W}$ contains an element $\rmZ'$ of minimum
  cardinality. Every proper subset $\rmZ^{*}\subsetneq\rmZ'$ satisfies
  $|\rmZ^{*}|<|\rmZ'|$ and $\rmZ^{*}\subseteq\rmZ$, so by the minimality of
  $\rmZ'$ we have $\rmZ^{*}\notin\mathcal{W}$, i.e.\
  $\neg\,\kassoc{\rmS}{Y,X}{\rmZ^{*}}$. Together with
  $\kassoc{\rmS}{Y,X}{\rmZ'}$, this is by \Cref{eq:k-dep-dep} exactly
  $\kdep{\rmS}{Y,X}{\rmZ'}$. As $\rmZ'\subseteq\rmZ$, this establishes
  \Cref{eq:11}.
\end{proof}

\kdepinter*
\begin{proof}
  Assume $\kdep{\emptyset}{Y,X}{\rmZ}$. For the variable $X$, the
  inter-association clause is $\kassoc{\emptyset}{Y,X}{\rmZ}$, i.e.\
  $Y\dep{\pr}X\mid\rmZ$, which is exactly the first conjunct of
  \Cref{eq:k-dep-dep} and hence immediate.

  Now suppose, towards a contradiction, that the clause fails for some
  $Z\in\rmZ$, i.e.\ $Y \indep{\pr} Z \mid X \cup (\rmZ \setminus Z)$.
  Since $\rmZ\setminus Z$ is a proper subset of $\rmZ$ and the
  separating set is empty, the second conjunct of \Cref{eq:k-dep-dep}
  gives $Y \indep{\pr} X \mid \rmZ \setminus Z$. Applying the
  contraction and then weak union graphoid axioms of
  \Cref{def:graphoid},
  \begin{equation}
    \begin{aligned}
      &(Y \indep{\pr} Z \mid X \cup (\rmZ \setminus Z))
        \wedge (Y \indep{\pr} X \mid \rmZ \setminus Z)\\
      &\implies Y \indep{\pr} \{X, Z\} \mid \rmZ \setminus Z
      \implies Y \indep{\pr} X \mid \rmZ,
    \end{aligned}\label{eq:k-dep-inter-eq4}
  \end{equation}
  contradicting $Y\dep{\pr}X\mid\rmZ$. Hence
  $Y\dep{\pr}Z\mid X\cup(\rmZ\setminus Z)$ for every $Z\in\rmZ$, which
  together with $Y\dep{\pr}X\mid\rmZ$ is precisely
  $\kinter{\emptyset}{Y,\rmZ\cup X}$.
\end{proof}

\section{Correctness of {\method}}\label{apx:proofs-thm}
\citest*

\begin{proof}
  The function \textit{find\_cond} first iterates through all the
  subsets $\rmX'\in\powerleq{k-|\rmX|}{\rmS}$ and for each $\rmX'$ where
  $Y \dep{\pre} X \mid \rmX\cup\rmX'$ (line 9 in \textit{find\_cond}), it passes
  $\rmX'$ to the function \textit{no\_seps} (line 10 in
  \textit{find\_cond}) to determine if there are any subsets
  $\rmS'\in\powerleq{l}{\rmS\setminus\rmX'}$ that results in
  $Y \indep{\pre} X \mid \rmX\cup\rmX'\cup \rmS'$ (line 4 in \textit{no\_seps}).
  As soon as \textit{no\_seps} finds such a $\rmS'$, it returns
  \texttt{False} (line 4 in \textit{no\_seps}), and \textit{find\_cond}
  will need to go to the next iteration. Otherwise, if \textit{no\_seps}
  is unable to find any such subset $\rmS'$, it returns \texttt{True}
  (line 5 in \textit{no\_seps}), and \textit{find\_cond} then
  immediately returns $\rmZ=\rmX'\cup\rmX$ (line 11 in
  \textit{find\_cond}) as there are no $l$-bounded subsets
  $\rmS'\in\powerleq{l}{\rmS\setminus\rmZ}$ that render $Y$ and $X$
  conditionally independent given
  $\rmZ\cup\rmS'$. Here we use that the external set $\rmX$ is disjoint from
  $\rmS$, so $\rmS\setminus\rmZ=\rmS\setminus\rmX'$; and since
  $\emptyset\in\powerleq{l}{\rmS\setminus\rmX'}$, the dependence check
  $Y\dep{\pre}X\mid\rmZ$ on line 9 is precisely the $\rmS'=\emptyset$
  instance tested by \textit{no\_seps}, so a returned $\rmZ$ satisfies
  \Cref{eq:25} in full. However, if \textit{find\_cond} completes all its iteration
  without returning early, then we know that all subset $\rmX'$ has some
  $l$-bounded subset $\rmS'$ that will cause
  $Y\indep{\pre}X\mid\rmX'\cup\rmX\cup\rmS'$. Therefore
  \textit{find\_cond} returns the failure value $\bot$ in this case
  (line 12 in \textit{find\_cond}).
\end{proof}

\assocmine*
\begin{proof}
  First consider a call that returns a nonempty set.
  \textit{find\_inter\_dep} does so only by returning
  \textit{make\_strict}$(Y,X,\rmZ,\rmS,k,l)$ after a call
  \textit{find\_cond}$(Y,X,\rmX,\rmS,k,l)$ has returned some
  $\rmZ\neq\bot$, for some $X\in\verts\setminus(Y\cup\rmS\cup\rmX)$
  (as $\rmS$ is passed unchanged through every recursive call, any such
  $X$ satisfies $X\in\verts\setminus(Y\cup\rmS)=\verts'$, matching the
  theorem's candidate set). By
  \Cref{thm:citest}, $\rmZ=\rmX\cup\rmX'$ for some
  $\rmX'\in\powerleq{k-|\rmX|}{\rmS}$ such that
  $Y\dep{\pre}X\mid\rmZ\cup\rmS'$ holds for all
  $\rmS'\in\powerleq{l}{\rmS\setminus\rmX'}$. As the dependence set $\rmX$ is
  disjoint from $\rmS$, we have $\rmS\setminus\rmX'=\rmS\setminus\rmZ$, so
  $Y\dep{\pre}X\mid\rmZ\cup\rmS'$ holds for all
  $\rmS'\in\powerleq{l}{\rmS\setminus\rmZ}$ with
  $\rmZ\in\powerleq{k}{\verts\setminus\{Y,X\}}$---the estimated $l$-bounded
  association claimed in the theorem. Since
  \textit{make\_strict} only removes variables from $\rmZ$, the returned
  set is $X\cup\rmZ_{R}$ for some $\rmZ_{R}\subseteq\rmZ$, as claimed.

  Next consider a call that returns the empty set. Here every internal
  \textit{find\_cond} call must have failed, so \textit{find\_inter\_dep}
  traverses the full enumeration tree of the dependence sets
  $\rmX\in\powerleq{k}{\verts\setminus(Y\cup\rmS)}$ (lines 13--17 of
  \textit{find\_inter\_dep}), and for each visits every candidate
  $X\in\verts\setminus(Y\cup\rmS\cup\rmX)$ through
  \textit{find\_cond}$(Y,X,\rmX,\rmS,k,l)$, which by \Cref{thm:citest}
  searches all $\rmX'\in\powerleq{k-|\rmX|}{\rmS}$. Every conditioning
  set $\rmZ\in\powerleq{k}{\verts\setminus\{Y,X\}}$ decomposes uniquely
  into its part outside and its part inside the candidate blanket,
  $\rmZ=(\rmZ\setminus\rmS)\cup(\rmZ\cap\rmS)=\rmX\cup\rmX'$, so this
  enumeration covers every pair $(X,\rmZ)$ with $X\in\verts'$ and
  $\rmZ\in\powerleq{k}{\verts\setminus\{Y,X\}}$. By \Cref{thm:citest},
  the call for $(X,\rmX)$ returns $\bot$ exactly when every such $\rmX'$
  admits a separator $\rmS'\in\powerleq{l}{\rmS\setminus\rmX'}$ with
  $Y\indep{\pre}X\mid\rmX\cup\rmX'\cup\rmS'$. Therefore
  \textit{find\_inter\_dep} returns the empty set if and only if
  \Cref{eq:fid-empty} holds. The distinguished value $\bot$ lets the
  callers tell such a failure from a success that returns the empty
  conditioning set.
\end{proof}

\algkomb*
\begin{proof}
  Throughout we assume that the \gls{ci} decisions are accurate, so
  $\pre$ and $\pr$ agree on every tested statement.

  In the grow phase, every nonempty set returned by
  \textit{find\_inter\_dep} contains a variable $X\notin\rmS$ by
  \Cref{thm:assocmine}, so each iteration of the repeat loop strictly
  grows $\rmS$ and the phase terminates. Upon termination
  \textit{find\_inter\_dep} returned the empty set, so by
  \Cref{thm:assocmine} the condition in \Cref{eq:fid-empty} holds in
  $\pre$, and by accuracy in $\pr$. Fix any
  $X\in\verts\setminus(\rmS\cup Y)$. For every
  $\rmZ\in\powerleq{k}{\verts\setminus\{Y,X\}}$, \Cref{eq:fid-empty}
  supplies an $\rmS'\in\powerleq{l}{\rmS\setminus\rmZ}
  \subseteq\powerleq{l}{\rmS}$ with $Y\indep{\pr}X\mid\rmZ\cup\rmS'$;
  the $l$-bounded separator assumption (\Cref{def:bound-sep}) then
  lifts this, for every such $\rmZ$, to the unbounded statement that
  some $\rmS'\subseteq\rmS$ gives $Y\indep{\pr}X\mid\rmZ\cup\rmS'$,
  which is the antecedent of $k$-order faithfulness
  (\Cref{eq:k-faith-indep}). Hence $Y\indep{\gr}X\mid\rmS$ for every
  $X\in\verts\setminus(\rmS\cup Y)$. Since the parents and children of
  $Y$---and, for an undirected $\gr$, its neighbours---can never be
  separated from $Y$, they must all lie in $\rmS$. Any spouse $X$
  shares a child $C\in\text{ch}(Y)\subseteq\rmS$, so the collider path
  $Y\rightarrow C\leftarrow X$ is active given $\rmS$; hence
  $Y\dep{\gr}X\mid\rmS$, which forces $X\in\rmS$ as well. Therefore
  $\mbs{\gr}(Y)\subseteq\rmS$, establishing \Cref{eq:27}.

  During the shrink phase we maintain the invariant
  $\mbs{\gr}(Y)\subseteq\rmS$, which holds on entry by \Cref{eq:27}.
  First, we show that no true member is ever removed. For
  $X\in\mbs{\gr}(Y)$, the in-blanket witness hypothesis of
  \Cref{thm:algkomb} supplies a set
  $\rmZ\in\powerleq{k}{\mbs{\gr}(Y)\setminus X}
  \subseteq\powerleq{k}{\rmS\setminus X}$ such that
  $Y\dep{\pr}X\mid\rmZ\cup\rmS'$ holds for all
  $\rmS'\in\powerleq{l}{(\rmS\setminus X)\setminus\rmZ}$. By accuracy and \Cref{thm:citest}
  (with candidate blanket $\rmS\setminus X$, external dependence set
  $\rmX=\emptyset$, and $\rmX'=\rmZ$), the call
  \textit{find\_cond}$(Y,X,\emptyset,\rmS\setminus X,k,l)$ does not
  return $\bot$, so $X$ is never selected for removal (a member whose
  only witness is $\rmZ=\emptyset$ makes \textit{find\_cond} return the
  empty set rather than $\bot$).

  Next, we show that every non-member is removed once selected. Let
  $X\in\rmS\setminus\mbs{\gr}(Y)$; by the invariant
  $\mbs{\gr}(Y)\subseteq\rmS\setminus X$. We first observe that
  $\mbs{\gr}(Y)$ separates $Y$ from $X$ even after conditioning on any
  further set $\rmZ\subseteq\verts\setminus\{Y,X\}$. Since
  $X\notin\mbs{\gr}(Y)$, $X$ is not adjacent to $Y$, so every path
  between them has at least one interior vertex; write such a path as
  $Y=N_{0}\sim N_{1}\sim\cdots\sim N_{m}=X$ with $m\geq 2$. Its first
  interior vertex $N_{1}$ is adjacent to $Y$ and hence lies in
  $\mbs{\gr}(Y)$. If $N_{1}$ is a non-collider on the path, it blocks
  the path. If instead $N_{1}$ is a collider
  $Y\rightarrow N_{1}\leftarrow N_{2}$, then $N_{1}$ is a common child of
  $Y$ and its successor $N_{2}$ on the path, so $N_{2}$ is a spouse of
  $Y$ and hence $N_{2}\in\mbs{\gr}(Y)$. In particular $N_{2}\neq X$ (as
  $X\notin\mbs{\gr}(Y)$), so $N_{2}$ is a genuine interior vertex; and
  since the edge $N_{2}\rightarrow N_{1}$ points out of $N_{2}$, it is a
  non-collider on the path and blocks it. For an undirected $\gr$, $N_{1}$ is instead a neighbour of
  $Y$ and blocks the path directly. As each path is thus blocked at a
  conditioned vertex of $\mbs{\gr}(Y)$ irrespective of
  $\rmZ$~\citep{pearl1988,lauritzen1996},
  $Y\indep{\gr}X\mid\mbs{\gr}(Y)\cup\rmZ$ for every
  $\rmZ\in\powerleq{k}{\verts\setminus\{Y,X\}}$.

  The Global Markov
  Property transfers this to $\pr$ with the separator
  $\rmS'\coloneq\mbs{\gr}(Y)\setminus\rmZ\subseteq\rmS\setminus X$, so
  the unbounded side of \Cref{def:bound-sep}---at the separating set
  $\rmS\setminus X$---holds for every such $\rmZ$. The biconditional
  then yields some $\rmS'\in\powerleq{l}{\rmS\setminus X}$; discarding
  any overlap with $\rmZ$ leaves $\rmZ\cup\rmS'$ unchanged and gives
  $\rmS'\in\powerleq{l}{(\rmS\setminus X)\setminus\rmZ}$ with
  $Y\indep{\pr}X\mid\rmZ\cup\rmS'$. By accuracy and \Cref{thm:citest},
  \textit{find\_cond}$(Y,X,\emptyset,\rmS\setminus X,k,l)$ returns
  $\bot$, so $X$ is removed once selected.

  Each iteration of the while loop removes one variable, so the phase
  terminates; it exits only when no $X\in\rmS$ yields $\bot$, whence
  $\rmS\setminus\mbs{\gr}(Y)=\emptyset$ and, with the invariant,
  $\rmS=\mbs{\gr}(Y)$, establishing \Cref{eq:28}.
\end{proof}

\section{Complexity of {\method}}\label{apx:proof-complex}

\kombcomplexity*
\begin{proof}
  We charge $O(N)$ to each \gls{ci} test and bound the running time by
  analysing the three sub-algorithms of \Cref{sec:algo} from the
  innermost outwards.

  First consider \textit{find\_cond} (\Cref{alg:ci-test}) for a fixed
  candidate variable $X$ and external dependence set $\rmX$ with
  $|\rmX|=i$. It iterates over every subset
  $\rmX'\in\powerleq{k-i}{\rmS}$, of which there are at most
  $\sum_{i'=0}^{k-i}\binom{K}{i'}$ since $|\rmS|\leq K$. For each $\rmX'$
  it performs one \gls{ci} test and, whenever that test indicates a
  dependence, calls \textit{no\_seps}, which iterates over every
  separator $\rmS'\in\powerleq{l}{\rmS\setminus\rmX'}$ and performs one
  further \gls{ci} test per separator. Writing $|\rmX'|=i'$, there are at
  most $\sum_{j=0}^{l}\binom{K-i'}{j}$ such separators, so a single
  call to \textit{find\_cond} costs
  \begin{equation}
    \label{eq:cost-findcond}
    O\!\left(
      \sum_{i'=0}^{k-i}\binom{K}{i'}
      \sum_{j=0}^{l}\binom{K-i'}{j}\, N
    \right),
  \end{equation}
  where the convention $\binom{n}{m}=0$ for $n<m$ accounts for the
  separator set $\rmS\setminus\rmX'$ being exhausted.

  Next consider \textit{find\_inter\_dep} (\Cref{alg:assoc-mine}). It
  explores the enumeration tree of $\verts\setminus Y$ up to subsets of
  cardinality $k$ (lines 13--17 of \textit{find\_inter\_dep}), and hence
  visits at most $\sum_{i=0}^{k}\binom{n}{i}$ dependence sets
  $\rmX\in\powerleq{k}{\verts\setminus Y}$ (a loose bound on the exact
  count $\sum_{i=0}^{k}\binom{n-1}{i}$, as $|\verts\setminus Y|=n-1$). For every such $\rmX$ it
  loops over the $O(n)$ candidate variables
  $X\in\verts\setminus(Y\cup\rmS\cup\rmX)$ and invokes
  \textit{find\_cond}. Multiplying these two factors by
  \Cref{eq:cost-findcond} bounds the cost of a single association-mining
  sweep by
  \begin{equation}
    \label{eq:cost-sweep}
    O\!\left(
      n \sum_{i=0}^{k}\binom{n}{i}
      \left[
        \sum_{i'=0}^{k-i}\binom{K}{i'}
        \sum_{j=0}^{l}\binom{K-i'}{j}\, N
      \right]
    \right).
  \end{equation}
  Each successful sweep additionally invokes \textit{make\_strict}
  exactly once; it iterates over the found dependence set $\rmZ$---of
  size at most $\min(k,n-2)$, since $\rmZ\subseteq\verts\setminus\{Y,X\}$
  and $|\rmZ|\leq k$---and performs a single \textit{no\_seps} call per
  element, at cost
  $O\!\left(\min(k,n-2)\sum_{j=0}^{l}\binom{K}{j}\,N\right)$. This is dominated by
  the per-sweep \textit{find\_cond} total already counted in
  \Cref{eq:cost-sweep}---whose bracketed factor already contains the
  term $\sum_{j=0}^{l}\binom{K}{j}$ (at $i'=0$) multiplied by the outer
  factor $n\sum_{i=0}^{k}\binom{n}{i}\geq n\geq\min(k,n-2)$---and is
  therefore absorbed into the bound.

  Finally, consider {\method} itself (\Cref{alg:komb}). Each successful
  call to \textit{find\_inter\_dep} in the grow phase returns a
  nonempty set $X\cup\rmZ$ that is added to $\rmS$, so the candidate
  Markov blanket grows by at least one variable per successful call; as
  $|\rmS|\leq K$ throughout, the grow phase performs at most $O(K)$
  association-mining sweeps. The shrink phase re-evaluates its guard after
  each removal, calling \textit{find\_cond} at most $O(K^{2})$ times over the
  at most $K$ removals; as $K\leq n-1$, this cost is dominated by that of the
  grow phase. Multiplying \Cref{eq:cost-sweep}
  by these $O(K)$ sweeps yields the bound in \Cref{eq:complexity},
  completing the proof.
\end{proof}

\section{Boolean Functions for the Synthetic Bayesian Networks}\label{apx:bool-func}
\begin{description}[align=right,labelwidth=7em]
\item[parity:]     
  $f(X_{1},\ldots,X_{3}) = (\sum_{i=1}^{3}X_{i}) \pmod{2}$
\item[ex-$s$:] 
  $f(X_{1},\ldots,X_{3}; s) = \bigl[(\sum_{i=1}^{3}X_{i}) = s\bigr]$
\item[and:]       
  $f(X_{1},\ldots,X_{3}) = \bigl[(\sum_{i=1}^{3}X_{i}) = 3\bigr]$
\item[or:]         
  $f(X_{1},\ldots,X_{3}) = \bigl[(\sum_{i=1}^{3}X_{i}) \geq 1\bigr]$
\end{description}
where $s\in\{1,2\}$.

\section{Separator Sizes in the Benchmark Datasets}\label{apx:dsep}
To gauge how restrictive the $l$-bounded separator assumption
(\Cref{def:bound-sep}) is in practice, we examine the benchmark networks
of \Cref{sec:exp-real}. We use every variable in turn as a target and,
for every pair formed by a target and a variable outside its Markov
blanket, compute a minimal d-separator drawn from the true Markov
blanket. \Cref{tab:dsep} reports the mean and worst-case (maximum)
separator size and the fraction of pairs whose separator has size at
most three. Separators are small across all four networks---the mean
never exceeds $2.19$ and size-$\leq\!3$ separators cover at least $83\%$
of pairs---and no pair had to be skipped, i.e.\ a separator within the
true Markov blanket always existed. The bound is loosest on Barley,
where the largest separator reaches size $9$; this is also the network
on which {\method} is the most expensive and least accurate
(\Cref{sec:exp-real}), illustrating that the difficulty of \gls{mb}
discovery tracks the separator sizes the assumption must accommodate.
Overall, a modest separator bound $l$ already suffices to capture most
of the conditional independencies needed to recover the Markov blanket
on these networks.

\begin{table}[h]
  \centering
  \caption{Minimal d-separator sizes restricted to the true Markov
    blanket, computed using every variable of each benchmark network as
    a target. ``Targets'' is the number of targets (all variables) and
    ``Pairs'' the number of target--non-blanket pairs; ``Mean sep.\ size''
    and ``Max'' are the mean and maximum separator size; and
    ``Cover $\leq 3$'' is the fraction whose minimal separator has size
    at most three.}
  \label{tab:dsep}
  \resizebox{\columnwidth}{!}{\begin{tabular}{llllll}
\toprule
Network & Targets & Pairs & Mean sep.\ size & Max & Cover $\leq 3$ \\
\midrule
Alarm1 & 37 & 1202 & 0.621 & 3 & 1.000 \\
Barley & 48 & 2004 & 1.739 & 9 & 0.835 \\
Insurance & 27 & 562 & 2.190 & 7 & 0.925 \\
Mildew & 35 & 1030 & 1.065 & 6 & 0.976 \\
\bottomrule
\end{tabular}
}
\end{table}

\section{Ablation}\label{apx:ablation}
For completeness, we report the full results of the experiments in
\Cref{sec:exp-real} for every method and every $(k,l)$ configuration of
{\method}. \Cref{tab:real-full} gives the performance of all methods,
while \Cref{tab:runtime-full} gives their runtimes in seconds. In both
tables, ``--'' denotes a configuration that failed to complete within the
allotted time budget.

\Cref{tab:kl-ablation} summarises the F1 score of {\method} across all
synthetic and benchmark datasets for every $(k,l)$ configuration. F1
generally improves from $k=0$ to $k=2$---higher orders expose
dependencies that lower orders miss---but the gains flatten or reverse
beyond $k=2$ as the larger conditioning sets incur both higher runtime
(\Cref{tab:runtime-full}) and more error-prone \gls{ci} tests. Setting
$l=k$ with $k\leq 2$ therefore offers a robust default, the heuristic we
adopt in \Cref{sec:complexity}. For a fixed $k$, raising $l$ beyond $k$
does not consistently improve F1 (\Cref{tab:kl-ablation}), which
motivates simply setting $l=k$ rather than tuning $l$ separately.

\begin{table*}[h]
  \centering
  \caption{F1 score across datasets for each $(k,l)$ configuration of {\method}.}
  \label{tab:kl-ablation}
  \setlength{\tabcolsep}{5pt}
\begin{tabular}{lccccccccc}
\toprule
Method & parity & ex-1 & ex-2 & and & or & Alarm1 & Barley & Insurance & Mildew \\
\midrule

kOMB
& 0.0250 & 0.2825 & 0.2675 & 0.4100 & 0.3500 & -- & -- & -- & -- \\
{\scriptsize $(k=0,l=1)$} & {\scriptsize $\pm$0.1104} & {\scriptsize $\pm$0.3720} & {\scriptsize $\pm$0.3737} & {\scriptsize $\pm$0.4050} & {\scriptsize $\pm$0.3508} & -- & -- & -- & -- \\

kOMB
& 0.0250 & 0.2825 & 0.2675 & 0.4100 & 0.3500 & -- & -- & -- & -- \\
{\scriptsize $(k=0,l=2)$} & {\scriptsize $\pm$0.1104} & {\scriptsize $\pm$0.3720} & {\scriptsize $\pm$0.3737} & {\scriptsize $\pm$0.4050} & {\scriptsize $\pm$0.3508} & -- & -- & -- & -- \\

kOMB
& 0.0250 & 0.2825 & 0.2675 & 0.4100 & 0.3500 & -- & -- & -- & -- \\
{\scriptsize $(k=0,l=3)$} & {\scriptsize $\pm$0.1104} & {\scriptsize $\pm$0.3720} & {\scriptsize $\pm$0.3737} & {\scriptsize $\pm$0.4050} & {\scriptsize $\pm$0.3508} & -- & -- & -- & -- \\

kOMB
& 0.0500 & 0.6750 & 0.6750 & \textbf{0.6175} & 0.5600 & 0.7804 & \textbf{0.5158} & 0.7136 & \textbf{0.6220} \\
{\scriptsize $(k=1,l=1)$} & {\scriptsize $\pm$0.2207} & {\scriptsize $\pm$0.4743} & {\scriptsize $\pm$0.4743} & {\scriptsize $\pm$0.3741} & {\scriptsize $\pm$0.4349} & {\scriptsize $\pm$0.1145} & {\scriptsize $\pm$0.1964} & {\scriptsize $\pm$0.1328} & {\scriptsize $\pm$0.1595} \\

kOMB
& 0.0500 & 0.6750 & 0.6750 & \textbf{0.6175} & 0.5600 & 0.8187 & 0.3128 & \textbf{0.7336} & 0.4360 \\
{\scriptsize $(k=1,l=2)$} & {\scriptsize $\pm$0.2207} & {\scriptsize $\pm$0.4743} & {\scriptsize $\pm$0.4743} & {\scriptsize $\pm$0.3741} & {\scriptsize $\pm$0.4349} & {\scriptsize $\pm$0.1944} & {\scriptsize $\pm$0.1582} & {\scriptsize $\pm$0.1624} & {\scriptsize $\pm$0.1394} \\

kOMB
& 0.0500 & 0.6750 & 0.6750 & \textbf{0.6175} & 0.5600 & 0.7923 & 0.2498 & 0.7076 & 0.3726 \\
{\scriptsize $(k=1,l=3)$} & {\scriptsize $\pm$0.2207} & {\scriptsize $\pm$0.4743} & {\scriptsize $\pm$0.4743} & {\scriptsize $\pm$0.3741} & {\scriptsize $\pm$0.4349} & {\scriptsize $\pm$0.1885} & {\scriptsize $\pm$0.1140} & {\scriptsize $\pm$0.1304} & {\scriptsize $\pm$0.1558} \\

kOMB
& \textbf{1.0000} & \textbf{1.0000} & \textbf{1.0000} & \textbf{0.6175} & \textbf{0.7100} & 0.7775 & 0.4634 & 0.7249 & 0.5059 \\
{\scriptsize $(k=2,l=1)$} & {\scriptsize $\pm$0.0000} & {\scriptsize $\pm$0.0000} & {\scriptsize $\pm$0.0000} & {\scriptsize $\pm$0.3741} & {\scriptsize $\pm$0.3842} & {\scriptsize $\pm$0.1511} & {\scriptsize $\pm$0.1774} & {\scriptsize $\pm$0.1464} & {\scriptsize $\pm$0.1458} \\

kOMB
& \textbf{1.0000} & \textbf{1.0000} & \textbf{1.0000} & \textbf{0.6175} & \textbf{0.7100} & \textbf{0.8209} & 0.3128 & \textbf{0.7336} & 0.4360 \\
{\scriptsize $(k=2,l=2)$} & {\scriptsize $\pm$0.0000} & {\scriptsize $\pm$0.0000} & {\scriptsize $\pm$0.0000} & {\scriptsize $\pm$0.3741} & {\scriptsize $\pm$0.3842} & {\scriptsize $\pm$0.1914} & {\scriptsize $\pm$0.1582} & {\scriptsize $\pm$0.1624} & {\scriptsize $\pm$0.1394} \\

kOMB
& \textbf{1.0000} & \textbf{1.0000} & \textbf{1.0000} & \textbf{0.6175} & \textbf{0.7100} & 0.7967 & 0.2498 & 0.7058 & 0.3726 \\
{\scriptsize $(k=2,l=3)$} & {\scriptsize $\pm$0.0000} & {\scriptsize $\pm$0.0000} & {\scriptsize $\pm$0.0000} & {\scriptsize $\pm$0.3741} & {\scriptsize $\pm$0.3842} & {\scriptsize $\pm$0.1828} & {\scriptsize $\pm$0.1140} & {\scriptsize $\pm$0.1311} & {\scriptsize $\pm$0.1558} \\

kOMB
& -- & -- & -- & -- & -- & 0.7721 & -- & 0.7227 & 0.5243 \\
{\scriptsize $(k=3,l=1)$} & -- & -- & -- & -- & -- & {\scriptsize $\pm$0.1599} & -- & {\scriptsize $\pm$0.1552} & {\scriptsize $\pm$0.1075} \\

kOMB
& -- & -- & -- & -- & -- & 0.8001 & -- & 0.7288 & 0.4194 \\
{\scriptsize $(k=3,l=2)$} & -- & -- & -- & -- & -- & {\scriptsize $\pm$0.1948} & -- & {\scriptsize $\pm$0.1471} & {\scriptsize $\pm$0.1414} \\

kOMB
& -- & -- & -- & -- & -- & 0.7967 & -- & 0.7058 & 0.3726 \\
{\scriptsize $(k=3,l=3)$} & -- & -- & -- & -- & -- & {\scriptsize $\pm$0.1828} & -- & {\scriptsize $\pm$0.1311} & {\scriptsize $\pm$0.1558} \\
\bottomrule
\end{tabular}

\end{table*}

\begin{table}[h]
  \centering
  \caption{Full results on benchmark datasets.}
  \label{tab:real-full}
  \setlength{\tabcolsep}{5pt}
\begin{tabular}{lcccc}
\toprule
Method & Alarm1 & Barley & Insurance & Mildew \\
\midrule

BAMB
& 0.6742 & 0.3385 & 0.6588 & 0.4945 \\
& {\scriptsize $\pm$0.1771} & {\scriptsize $\pm$0.1226} & {\scriptsize $\pm$0.0797} & {\scriptsize $\pm$0.1084} \\

GS
& 0.3434 & 0.2113 & 0.4615 & 0.2867 \\
& {\scriptsize $\pm$0.1503} & {\scriptsize $\pm$0.0807} & {\scriptsize $\pm$0.1477} & {\scriptsize $\pm$0.1581} \\

HITON\_MB
& 0.7636 & 0.3400 & 0.6746 & 0.4918 \\
& {\scriptsize $\pm$0.1861} & {\scriptsize $\pm$0.1284} & {\scriptsize $\pm$0.1594} & {\scriptsize $\pm$0.1487} \\

IAMB
& 0.6931 & 0.3373 & 0.6293 & 0.4677 \\
& {\scriptsize $\pm$0.1827} & {\scriptsize $\pm$0.1223} & {\scriptsize $\pm$0.0778} & {\scriptsize $\pm$0.1224} \\

LRH
& 0.6813 & 0.3373 & 0.6281 & 0.4674 \\
& {\scriptsize $\pm$0.1777} & {\scriptsize $\pm$0.1223} & {\scriptsize $\pm$0.0782} & {\scriptsize $\pm$0.1230} \\

MMMB
& 0.7692 & 0.3298 & 0.6978 & 0.4861 \\
& {\scriptsize $\pm$0.1874} & {\scriptsize $\pm$0.1274} & {\scriptsize $\pm$0.1405} & {\scriptsize $\pm$0.1810} \\

PCMB
& 0.7371 & 0.2284 & 0.6299 & 0.4156 \\
& {\scriptsize $\pm$0.2127} & {\scriptsize $\pm$0.1199} & {\scriptsize $\pm$0.1291} & {\scriptsize $\pm$0.1445} \\

STMB
& 0.6598 & 0.2145 & 0.5648 & 0.2734 \\
& {\scriptsize $\pm$0.1627} & {\scriptsize $\pm$0.1281} & {\scriptsize $\pm$0.1324} & {\scriptsize $\pm$0.1765} \\

kOMB
& 0.7804 & \textbf{0.5158} & 0.7136 & \textbf{0.6220} \\
{\scriptsize $(k=1,l=1)$} & {\scriptsize $\pm$0.1145} & {\scriptsize $\pm$0.1964} & {\scriptsize $\pm$0.1328} & {\scriptsize $\pm$0.1595} \\

kOMB
& 0.8187 & 0.3128 & \textbf{0.7336} & 0.4360 \\
{\scriptsize $(k=1,l=2)$} & {\scriptsize $\pm$0.1944} & {\scriptsize $\pm$0.1582} & {\scriptsize $\pm$0.1624} & {\scriptsize $\pm$0.1394} \\

kOMB
& 0.7923 & 0.2498 & 0.7076 & 0.3726 \\
{\scriptsize $(k=1,l=3)$} & {\scriptsize $\pm$0.1885} & {\scriptsize $\pm$0.1140} & {\scriptsize $\pm$0.1304} & {\scriptsize $\pm$0.1558} \\

kOMB
& 0.7775 & 0.4588 & 0.7249 & 0.5059 \\
{\scriptsize $(k=2,l=1)$} & {\scriptsize $\pm$0.1511} & {\scriptsize $\pm$0.1825} & {\scriptsize $\pm$0.1464} & {\scriptsize $\pm$0.1458} \\

kOMB
& \textbf{0.8209} & 0.3128 & \textbf{0.7336} & 0.4360 \\
{\scriptsize $(k=2,l=2)$} & {\scriptsize $\pm$0.1914} & {\scriptsize $\pm$0.1582} & {\scriptsize $\pm$0.1624} & {\scriptsize $\pm$0.1394} \\

kOMB
& 0.7967 & 0.2498 & 0.7058 & 0.3726 \\
{\scriptsize $(k=2,l=3)$} & {\scriptsize $\pm$0.1828} & {\scriptsize $\pm$0.1140} & {\scriptsize $\pm$0.1311} & {\scriptsize $\pm$0.1558} \\

kOMB
& 0.7721 & 0.0000 & 0.7227 & 0.0419 \\
{\scriptsize $(k=3,l=1)$} & {\scriptsize $\pm$0.1599} & {\scriptsize $\pm$0.0000} & {\scriptsize $\pm$0.1552} & {\scriptsize $\pm$0.1458} \\

kOMB
& 0.8001 & 0.0000 & 0.7288 & 0.3607 \\
{\scriptsize $(k=3,l=2)$} & {\scriptsize $\pm$0.1948} & {\scriptsize $\pm$0.0000} & {\scriptsize $\pm$0.1471} & {\scriptsize $\pm$0.1964} \\

kOMB
& 0.7967 & 0.0000 & 0.7058 & 0.3726 \\
{\scriptsize $(k=3,l=3)$} & {\scriptsize $\pm$0.1828} & {\scriptsize $\pm$0.0000} & {\scriptsize $\pm$0.1311} & {\scriptsize $\pm$0.1558} \\
\bottomrule
\end{tabular}

\end{table}

\begin{table}[h]
  \centering
  \caption{Full runtimes (in seconds) on benchmark datasets.}
  \label{tab:runtime-full}
  \setlength{\tabcolsep}{5pt}
\begin{tabular}{lcccc}
\toprule
Method & Alarm1 & Barley & Insurance & Mildew \\
\midrule

BAMB
& 3.517 & 1.149 & 5.615 & 0.797 \\
& {\scriptsize $\pm$4.305} & {\scriptsize $\pm$0.681} & {\scriptsize $\pm$5.358} & {\scriptsize $\pm$0.392} \\

GS
& 4.875 & 1.575 & 4.433 & 1.035 \\
& {\scriptsize $\pm$5.183} & {\scriptsize $\pm$0.606} & {\scriptsize $\pm$3.792} & {\scriptsize $\pm$0.341} \\

HITON\_MB
& 2.560 & 1.027 & 3.514 & 0.604 \\
& {\scriptsize $\pm$1.277} & {\scriptsize $\pm$0.298} & {\scriptsize $\pm$2.525} & {\scriptsize $\pm$0.218} \\

IAMB
& 2.527 & 1.104 & 1.650 & 0.943 \\
& {\scriptsize $\pm$1.710} & {\scriptsize $\pm$0.496} & {\scriptsize $\pm$0.967} & {\scriptsize $\pm$0.420} \\

LRH
& 89.877 & 5.071 & 45.401 & 1.833 \\
& {\scriptsize $\pm$93.285} & {\scriptsize $\pm$2.496} & {\scriptsize $\pm$24.653} & {\scriptsize $\pm$0.651} \\

MMMB
& 2.724 & 1.090 & 3.508 & 0.553 \\
& {\scriptsize $\pm$1.278} & {\scriptsize $\pm$0.374} & {\scriptsize $\pm$2.525} & {\scriptsize $\pm$0.128} \\

PCMB
& 15.515 & 4.982 & 27.213 & 3.066 \\
& {\scriptsize $\pm$10.203} & {\scriptsize $\pm$1.872} & {\scriptsize $\pm$21.695} & {\scriptsize $\pm$1.486} \\

STMB
& 5.305 & 1.895 & 14.658 & 0.561 \\
& {\scriptsize $\pm$5.855} & {\scriptsize $\pm$2.367} & {\scriptsize $\pm$33.396} & {\scriptsize $\pm$0.524} \\

kOMB
& 5.258 & 58.190 & 4.856 & 17.698 \\
{\scriptsize $(k=1,l=1)$} & {\scriptsize $\pm$2.921} & {\scriptsize $\pm$31.619} & {\scriptsize $\pm$2.041} & {\scriptsize $\pm$10.205} \\

kOMB
& 3.974 & 12.369 & 3.828 & 7.562 \\
{\scriptsize $(k=1,l=2)$} & {\scriptsize $\pm$2.468} & {\scriptsize $\pm$7.709} & {\scriptsize $\pm$2.235} & {\scriptsize $\pm$2.301} \\

kOMB
& 4.111 & 9.639 & 3.171 & 7.384 \\
{\scriptsize $(k=1,l=3)$} & {\scriptsize $\pm$2.791} & {\scriptsize $\pm$3.794} & {\scriptsize $\pm$1.707} & {\scriptsize $\pm$2.127} \\

kOMB
& 26.092 & 788.531 & 19.901 & 311.324 \\
{\scriptsize $(k=2,l=1)$} & {\scriptsize $\pm$14.668} & {\scriptsize $\pm$453.059} & {\scriptsize $\pm$7.226} & {\scriptsize $\pm$167.079} \\

kOMB
& 18.144 & 335.821 & 14.493 & 121.909 \\
{\scriptsize $(k=2,l=2)$} & {\scriptsize $\pm$7.566} & {\scriptsize $\pm$135.966} & {\scriptsize $\pm$7.280} & {\scriptsize $\pm$33.153} \\

kOMB
& 18.016 & 290.512 & 10.424 & 121.146 \\
{\scriptsize $(k=2,l=3)$} & {\scriptsize $\pm$9.135} & {\scriptsize $\pm$85.391} & {\scriptsize $\pm$4.290} & {\scriptsize $\pm$32.239} \\

kOMB
& 174.900 & 1800.000 & 96.149 & 1761.246 \\
{\scriptsize $(k=3,l=1)$} & {\scriptsize $\pm$83.504} & {\scriptsize $\pm$0.000} & {\scriptsize $\pm$46.913} & {\scriptsize $\pm$188.357} \\

kOMB
& 132.619 & 1800.000 & 58.350 & 1534.465 \\
{\scriptsize $(k=3,l=2)$} & {\scriptsize $\pm$46.566} & {\scriptsize $\pm$0.000} & {\scriptsize $\pm$21.745} & {\scriptsize $\pm$300.944} \\

kOMB
& 127.861 & 1800.000 & 53.452 & 1356.514 \\
{\scriptsize $(k=3,l=3)$} & {\scriptsize $\pm$49.652} & {\scriptsize $\pm$0.000} & {\scriptsize $\pm$18.285} & {\scriptsize $\pm$304.778} \\
\bottomrule
\end{tabular}

\end{table}

\end{document}